\documentclass{article}

\usepackage[utf8]{inputenc} 
\usepackage[sort]{natbib}
\usepackage[T1]{fontenc}    
\usepackage{hyperref}       
\usepackage{url}            
\usepackage{booktabs}       
\usepackage{amsmath,amsfonts,amsthm,bm,bbm,amssymb}
\usepackage{mathtools}
\usepackage{nicefrac}       
\usepackage{microtype}      
\usepackage{microtype}
\usepackage[pdftex]{graphicx}
\usepackage{url}
\usepackage{delimset}  
\usepackage{array}
\usepackage[ruled,vlined,linesnumbered]{algorithm2e}   
\usepackage[capitalize,noabbrev]{cleveref}

\usepackage{nicefrac}
\usepackage[dvipsnames]{xcolor}

\usepackage{fullpage}
\hypersetup{colorlinks=true,citecolor=Blue,linkcolor=Blue,urlcolor=black}

\newcommand{\Real}{\mathbb{R}}
\newcommand{\opt}{^{\star}}

\newcommand{\E}{\mathbb{E}}
\DeclarePairedDelimiterXPP\expect[2]{\E_{#1}}[]{}{\setargs{#2}}%
\DeclarePairedDelimiterXPP\varian[2]{\mathbb{V}_{#1}}[]{}{\setargs{#2}}%
\DeclarePairedDelimiterXPP\probability[2]{\mathbb{P}_{#1}}[]{}{\setargs{#2}}%
\NewDocumentCommand{\setargs}{>{\SplitArgument{1}{|}}m}
{\setargsaux#1}
\NewDocumentCommand{\setargsaux}{mm}
{\IfNoValueTF{#2}{#1}{\nonscript\,#1\nonscript\;\delimsize\vert\nonscript\:\allowbreak #2\nonscript\,}}
\DeclarePairedDelimiterXPP\expectaux[3]{\E_{#1}}[]{}{#2\nonscript\:\delimsize\vert\nonscript\:#3}%

\newcommand{\Ex}[1]{\expect*{}{#1} }

\renewcommand{\Pr}[1]{\probability*{}{#1} }
\newcommand{\tr}{^{\mathsf{T}}}
\newcommand{\probs}[1]{\Delta_{#1}}

\newcommand{\subjectto}{\operatorname{subject\;to} &}

\newcommand{\maximize}[1]{\operatorname*{maximize}_{#1} &}
\newcommand{\cs}{\\[1ex] & }

\newenvironment{mprog}{\begin{array}{>{\displaystyle}c>{\displaystyle}l>{\displaystyle}l>{\displaystyle}l}}{\end{array}}
\newcommand{\stc}{\\[1ex] \subjectto}

\DeclareMathOperator{\varo}{VaR}

\DeclareMathOperator{\evaro}{EVaR}

\renewcommand{\P}[1]{\mathbb{P}\left[ #1 \right]}

\newcommand{\var}[2]{\varo_{#1} \left[#2\right]}

\newcommand{\evar}[2]{\evaro_{#1} \left[#2\right]}

\renewcommand{\cite}[1]{\citep{#1}}

\theoremstyle{plain}
\newtheorem{theorem}{Theorem}[section]
\newtheorem{proposition}[theorem]{Proposition}
\newtheorem{lemma}[theorem]{Lemma}

\theoremstyle{definition}

\newtheorem{assumption}[theorem]{Assumption}
\theoremstyle{remark}

\newtheorem{example}{Example}

\title{Bayesian Regret Minimization in Offline Bandits}
\author{Marek Petrik, Guy Tennenholtz, Mohammad Ghavamzadeh}

\begin{document}

\maketitle


\begin{abstract}
We study how to make decisions that minimize Bayesian regret in offline linear bandits. Prior work suggests that one must take actions with maximum \emph{lower confidence bound} (LCB) on their reward. We argue that the reliance on LCB is inherently flawed in this setting and propose a new algorithm that directly minimizes upper bounds on the Bayesian regret using efficient conic optimization solvers. Our bounds build heavily on new connections to monetary risk measures. Proving a matching lower bound, we show that our upper bounds are tight, and by minimizing them we are guaranteed to outperform the LCB approach. Our numerical results on synthetic domains confirm that our approach is superior to LCB. 
\end{abstract}


\section{Introduction}

The problem of offline bandits is an important special case of offline reinforcement learning~(RL) in which the model consists of a single state and involves no state transitions~\cite{Hong2023}. Offline RL, a challenging research problem with a rich history, is inspired by the need to make reliable decisions when learning from a logged dataset~\cite{Lange2012, Rashidinejad2022}. Practical problems from recommendations to search to ranking can be modeled as offline bandits; see, for example, \citet{Hong2023} and references therein. Moreover, gaining a deeper theoretical understanding of offline bandits is a vital stepping stone in understanding the complete offline RL problem. 

We study the problem of minimizing the \emph{Bayesian regret} in the offline linear bandit setting. Bayesian regret differs substantially from its \emph{frequentist} counterpart. While frequentist regret assumes a fixed true model and studies algorithms' response to random datasets, Bayesian regret assumes a fixed dataset and studies algorithms' regret as a function of the true model. When provided with good priors, Bayesian methods offer sufficiently tight bounds to achieve excellent practical results~\cite{Lattimore2018, Gelman2014, Vaart2000}. As such, the strengths of Bayesian methods complement the scalability and simplicity of the frequentist algorithms.

Most prior work on Bayesian offline RL and bandits has adopted a form of pessimism that chooses actions with the highest \emph{lower confidence bounds}~(LCBs). These LCB-style algorithms compute a policy or action with the largest expected return (or reward), penalized by its uncertainty. The uncertainty penalty is computed from credible regions derived from the posterior distribution~\cite{Delage2009, Hong2023, Brown2020, Javed2021, Lobo2023}, and often gives rise to some form of robust optimization~\cite{Behzadian2021, Petrik2019}. LCB-style Bayesian algorithms are generally inspired by the success of this approach in frequentist settings, where LCB is typically computed from concentration inequalities~\cite{Xie2021, Rashidinejad2022, Jin2022, Ghosh2022, Cheng2022}.

In this paper, we propose a Bayesian regret minimization algorithm, called \texttt{BRMOB}, that takes a new approach to Bayesian offline bandits. Instead of adopting an LCB-style strategy, we directly minimize new regret \emph{upper bounds}. To derive these bounds, we reformulate the usual high-confidence objective as a Value-at-Risk~(VaR) of the epistemic uncertainty. Then, we bound the VaR by combining techniques from robust optimization and Chernoff analysis. Our bounds apply to both Gaussian and sub-Gaussian posteriors over the latent reward parameter. \texttt{BRMOB} minimizes the regret bounds efficiently using convex conic solvers. Finally, we also establish a matching lower bound that shows our upper bounds are tight. 

Compared with prior work in Bayesian offline bandits, \texttt{BRMOB} achieves tighter theoretical guarantees and better empirical performance. Two main innovations enable these improvements. First, \texttt{BRMOB} computes randomized policies. Our numerical results show that randomizing among actions results in hedging that can significantly reduce regret compared to deterministic policies. In contrast to \texttt{BRMOB}, most existing algorithms in Bayesian offline bandits~\cite{Hong2023} and Bayesian offline RL~\cite{ Delage2009, Petrik2019, Behzadian2021, Angelotti2021} are restricted to deterministic policies. Second, \texttt{BRMOB} is the only algorithm that explicitly minimizes Bayesian regret bounds. As discussed above, existing algorithms usually maximize the LCB on returns~\cite{Hong2023, Uehara2023}, which does not guarantee to reduce Bayesian regret. Similarly, recent algorithms that maximize the expected return~\cite{Steimle2021a, Su2023} are also not known to reduce Bayesian regret.

We also study the general suitability of LCB algorithms for minimizing Bayesian regret. While \texttt{BRMOB} significantly outperforms a particular LCB algorithm known as \texttt{FlatOPO}~\cite{Hong2023}, the more critical question is whether the general LCB approach is viable for Bayesian regret minimization. Using our new regret lower bounds, we answer this question negatively. More precisely, we show that penalizing reward uncertainty, the core of all LCB algorithms, is guaranteed to increase the algorithm's regret even in very simple problems. This is because actions with high uncertainty may also have a high upside and avoiding them increases regret. Therefore, we believe that explicit regret minimization, as in \texttt{BRMOB}, is a more promising future direction than LCB-style algorithms.

Bayesian regret minimization in offline bandits can also be framed as a chance-constrained optimization problem~\cite{Ben-Tal2009}. The recent chance-constrained optimization literature is mostly focused on constraints in which the function is concave or linear in the uncertain parameter~\cite{Gupta2019, Bertsimas2021}. However, the chance constraint in the Bayesian regret minimization problem is convex, preventing us from using these methods. Another non-concave chance-constrained optimization approach is to resort to scenario-based or sample-based methods~\cite{Calafiore2005, Nemirovski2006, Luedtke2008, Brown2020}. We briefly discuss these methods in \cref{sec:baseline-algorithms}. Such sample-based formulations are general, simple to implement and work well in practice. However, they scale poorly to large problems, provide no theoretical insights, and struggle to compute randomized policies. The closest result to our work is the Bernstein technique for bounding linear chance-constrained programs~\cite{Nemirovski2007, Pinter1989}, which is a special case of one of our bounds.

The paper is organized as follows. After introducing our notations and the popular risk-measure Value-at-Risk (VaR) in \cref{sec:prelim}, we formally define the problem of Bayesian regret minimization in offline bandits and connect it to minimizing VaR in \cref{sec:prel-offl-line}. In \cref{sec:algo}, we derive two new upper bounds on the Bayesian regret and propose our main algorithm, \texttt{BRMOB}, that is based on a simultaneous minimization of these two regret bounds. We also prove a lower bound on the regret that shows our upper bound is tight. In \cref{sec:bayes-regr-analys}, we first derive a regret bound for \texttt{BRMOB} in terms of problem parameters and show that it compares favorably with LCB-based algorithms. We then argue that the general LCB approach is unsuitable for minimizing Bayesian regret. Finally, in \cref{sec:numer-demonstr}, we compare \texttt{BRMOB}'s performance with three baseline algorithms on synthetic domains and show that it is preferable to LCB-style algorithms.


\section{Preliminaries}
\label{sec:prelim}

We begin by defining the notations we use throughout the paper. We use lower and upper case bold letters to denote vectors and matrices, such as $\bm{x}\in \Real^N$ and $\bm{A}\in \Real^{n\times n}$, and normal font for the elements of vectors and matrices, e.g.,~$x_i$. We define the weighted $\ell_2$-norm for any vector $\bm{x}\in\mathbb R^d$ and positive definite matrix $\bm{A}\in\mathbb R^{d\times d}$ as $\|\bm{x}\|_{\bm A}=\sqrt{\bm{x}\tr \bm{A}\bm{x}}$. We denote by $\Delta_k,\forall k\in\mathbb N$ the $k$-dimensional probability simplex, and by $\bm{I}$, $\bm{0}$, $\bm{1}$, and $\bm{1}_a$ the identity matrix, the zero vector, the one vector, and the one-hot vector all with appropriate dimensions. Random variables are adorned with a tilde and are not capitalized. For example, $\bm{\tilde{x}}$ represents a vector-valued random variable. Finally, we denote by $\Omega$ the probability space of a random variable. 

Suppose that $\tilde{x} \colon \Omega\rightarrow \mathbb R$ is a random variable that represents {\em costs}. Then, its \emph{value-at-risk}~(VaR) at a risk-level $\alpha\in [0,1)$ is usually defined as the largest lower bound on its $\alpha$-quantile~(e.g.,~\citealt[definition~4.45, and remark~A.20]{Follmer2016}):
\begin{subequations}
\begin{align}
 \label{eq:var-definition}
  \var{\alpha}{\tilde{x}}
  &= \inf \left\{ t \in \Real \mid \P{\tilde{x} > t} \le 1 - \alpha \right\} \\
  \label{eq:var-definition-sup}
 &= \sup\left\{ t \in \Real \mid \P{\tilde{x} \ge t} > 1 - \alpha\right\}. 
\end{align}
\end{subequations}
The definition of VaR in the literature depends on whether $\tilde{x}$ represents costs or rewards~\cite{Hau2023}. If $\Tilde{x}$ represents \emph{rewards}, maximizing $- \var{\alpha}{-\tilde{x}}$ is equivalent to minimizing $\var{\alpha}{\tilde{x}}$. For Gaussian random variables, $\tilde{x} \sim \mathcal{N}(\mu, \sigma^2)$, VaR has the following analytical form~\cite{Follmer2016}: 
\begin{equation} \label{eq:var-normal}
\var{\alpha}{\tilde{x}} \;=\; \mu + \sigma \cdot z_{\alpha}\,, \end{equation}
where $z_{\alpha}$ is the {\em $\alpha$-quantile} of $\mathcal N(0,1)$. 


\section{Bayesian Offline Bandits}
\label{sec:prel-offl-line}

In this section, we first formally define the problem of Bayesian regret minimization in offline bandits and connect it to monetary risk measures. We then describe two techniques that have been used in solving this problem.  



\subsection{Problem Definition} 
\label{sec:problem-definition}

Consider a stochastic linear bandit problem with $k\in\mathbb N$ arms (actions) from the set $\mathcal A = \{a_1,\ldots, a_k\}$. Each arm $a\in\mathcal A$ is associated with a $d$-dimensional feature vector $\bm{\phi}_a\in\mathbb R^d$ and its reward distribution has a mean $r(a;\bm{\theta}) = \bm{\phi}_a\tr \bm{\theta}$ for some unknown parameter $\bm{\theta}\in \mathbb R^d$. We define the feature matrix $\bm{\Phi}\in \Real^{d\times k}$ as $\bm{\Phi} = (\bm{\phi}_a)_{a\in \mathcal{A}}$. The goal of the agent is to learn a (possibly {\em randomized}) policy $\bm{\pi}\in\Delta_k$ to choose its actions accordingly. We denote by $\pi_a$ the probability according to which policy $\bm{\pi}$ selects an action $a\in\mathcal A$. The mean reward, or {\em value}, of a policy $\bm{\pi}$ is defined as  
\begin{equation} 
\label{eq:value-optimal}
\begin{aligned}
r(\bm{\pi}; \bm{\theta}) &= \Ex{r(\tilde{a}; \bm{\theta}) \mid \tilde{a} \sim \bm{\pi}} \\ 
&=  \sum_{a\in \mathcal{A}} \pi_a \cdot r(a;\bm{\theta})  
= \bm{\pi}\tr \bm{\Phi}\tr \bm{\theta}\,. 
\end{aligned}
\end{equation}
An optimal policy $\bm{\pi}^\star(\bm{\theta})$ is one that maximizes~\eqref{eq:value-optimal}. 

In the {\em offline} bandit setting, the agent only has access to a {\em logged dataset} $\tilde{\mathcal{D}} = \{(\tilde{a}_i,\tilde{y}_i)\}_{i=1}^n$, and is not capable of interacting further with the environment. Each pair $(\tilde{a}_i,\tilde{y}_i)$ in $\tilde{\mathcal{D}}$ consists of an action $\tilde{a}_i$ selected according to some arbitrary logging policy and a sampled reward $\tilde{y}_i$ from the reward distribution of action $\tilde{a}_i$. We use $D$ to refer to an instantiation of the random dataset $\tilde{\mathcal{D}}$. 

We take the {\em Bayesian} perspective in this paper and model our uncertainty about the reward parameter $\bm{\tilde{\theta}}\colon \Omega \to \Real^d$ by assuming it is a random variable with a known prior $P_{\bm{\tilde{\theta}}}(\bm{\theta})$. Therefore, all quantities that depend on $\bm{\tilde{\theta}}$ are also random. The logged dataset $D$ is used to derive the posterior density $P_{\bm{\tilde{\theta}}|D}(\bm{\theta})$ over the reward parameter. To streamline the notation, we denote by $\bm{\tilde{\theta}}_D:= (\bm{\tilde{\theta}} \mid \tilde{\mathcal{D}} = D)$ the random variable distributed according to this posterior distribution $P_{\bm{\tilde{\theta}}|D}$. We discuss the derivation of the posterior in \cref{sec:bayes-regr-analys}. 

As described above, in the Bayesian offline bandit setting we assume that the logged data $\tilde{\mathcal{D}}$ is fixed to some $D$ and the uncertainty is over the reward parameter $\bm{\tilde{\theta}}$. This is different than the {\em frequentist} offline setting in which the reward parameter is fixed, $\bm{\tilde{\theta}}=\bm{\theta}^\star$, and the randomness is over different datasets generated by the logging policy. 

In the Bayesian offline bandit setting, our goal is to compute a policy $\bm{\pi}\in\Delta_k$ that minimizes the \emph{high-confidence Bayesian regret} $\mathfrak R_\delta \colon \probs{k} \to \Real_+$ defined as
\begin{equation} 
\label{eq:regret-high-confidence def} 
\begin{gathered}
\mathfrak{R}_{\delta}(\bm{\pi}) :=  
\min \; \epsilon \qquad \text{subject to} \\ 
\mathbb{P}\left[\max_{a\in \mathcal{A}} \; r(a; \bm{\tilde{\theta}}_D) - r(\bm{\pi}; \bm{\tilde{\theta}}_D) \le \epsilon\right] \ge 1-\delta,
\end{gathered}
\end{equation}
where $\delta \in (0,\frac{1}{2})$ is the small error tolerance parameter. We also use $\alpha = 1-\delta$ to denote the confidence in the solution.

Note that~\eqref{eq:regret-high-confidence def} compares the value of a fixed policy $\bm{\pi}$ with the reward of an action (max action) that depends on the posterior random variable $\bm{\tilde{\theta}}_D$. Thus, one cannot expect to achieve a regret of zero. By taking a close look at the definition of regret in~\eqref{eq:regret-high-confidence def} and using the definition of VaR in~\eqref{eq:var-definition}, we may equivalently write our objective in~\eqref{eq:regret-high-confidence def} as
\begin{equation} 
\label{eq:regret-high-confidence-VaR}
\mathfrak{R}_{\delta}(\bm{\pi}) = \var{1-\delta}{ \max_{a\in \mathcal{A}} \; r(a; \bm{\tilde{\theta}}_D) - r(\bm{\pi}; \bm{\tilde{\theta}}_D) }\,.
\end{equation}

One could optimize other objectives besides the high-confidence regret in~\eqref{eq:regret-high-confidence-VaR}. Other objectives, such as maximizing the VaR of the reward, are easier to solve and \cref{sec:other-objectives} discusses them in greater detail.


\subsection{Baseline Algorithms} 
\label{sec:baseline-algorithms}

We now provide a brief description of two methods that have been used to solve Bayesian offline bandits (defined in~\cref{sec:problem-definition}) and closely related problems. 

\paragraph{Lower Confidence Bound (LCB)}
Pessimism to the uncertainty in the problem's parameter is the most common approach in offline decision-making problems, ranging from offline RL~\cite{Uehara2023, Rashidinejad2022, Xie2022c}, to robust RL~\cite{Petrik2019, Behzadian2021, Lobo2020}, and offline bandits~\cite{Hong2023}. In the case of offline bandits, this approach is compellingly simple and is known as maximizing a lower confidence bound, or LCB. The general recipe of the LCB algorithm for Gaussian and sub-Gaussian posteriors $\bm{\tilde{\theta}}_D$ is to simply choose the action $\hat{a}\in \mathcal{A}$ such that
\begin{equation} \label{eq:lcb-hong}
\hat{a}\in\arg \max_{a\in \mathcal{A}} \, \ell_{\beta}(a) := \left(  \bm{\mu}_n\tr \bm{\phi}_a - \beta \cdot \sqrt{\bm{\phi}_a\tr \bm{\Sigma}_n \bm{\phi}_a} \right) ,
\end{equation}
for some $\beta > 0$. The terms $\bm{\mu}_n\tr \bm{\phi}_a$ and $\sqrt{\bm{\phi}_a\tr \bm{\Sigma}_n \bm{\phi}_a}$ represent the posterior mean and standard deviation of $r(a,\bm{\tilde{\theta}}_D)=\bm{\phi}_a\tr\bm{\tilde{\theta}}_D$. The parameter $\beta$ is typically chosen to guarantee that $\ell_{\beta}(a)$ is a high-probability lower bound on the return of action $a \in \mathcal{A}$:
\[
  \P{\ell_{\beta}(a) \le r(a, \bm{\tilde{\theta}}_D)}
  \;\ge\;
  1-\delta.
\]
The \texttt{FlatOPO} algorithm~\cite{Hong2023} is a particular instance of the LCB approach to offline bandits that uses $\beta = \sqrt{5 d \log(\nicefrac{1}{\delta})}$ for Gaussian posteriors. When $\beta = 0$, we refer to an algorithm that implements~\eqref{eq:lcb-hong} as \texttt{Greedy}.

\paragraph{Scenario-based Methods}
Another natural approach to minimizing the Bayesian regret in offline bandits is to treat the optimization in~\eqref{eq:regret-high-confidence def} as a \emph{chance-constrained optimization} problem. The most general algorithm to solve chance-constrained optimization is to use \emph{scenario-based} techniques to minimize the regret $\mathfrak R_{\delta}(\bm{\pi})$~\cite{Calafiore2005, Nemirovski2006, Luedtke2008}. A typical scenario-based algorithm first approximates $\bm{\tilde{\theta}}_D$ with a \emph{discrete} random variable $\bm{\tilde{q}}$ constructed by sampling from its posterior $P_{\bm{\tilde{\theta}}|D}(\bm{\theta})$, and then computes a \emph{deterministic} policy by solving 
\begin{equation}
\label{eq:scenario-based}
 \arg \min_{\hat{a}\in \mathcal{A}} \; \var{1-\delta}{\max_{a\in \mathcal{A}} r(a; \bm{\tilde{q}}) - r(\hat{a}; \bm{\tilde{q}})}\,.
\end{equation}
The optimization in~\eqref{eq:scenario-based} can be solved by enumerating all the actions and computing the VaR of the discrete random variable within the brackets. The important question that has been extensively studied here is the number of samples needed to obtain a solution with high confidence~\cite{Nemirovski2007, Nemirovski2006}. The time complexity of this algorithm is a function of the number of samples and the desired confidence to guarantee a certain suboptimality of the solution; we refer the interested reader to \citet{Nemirovski2007} for a detailed analysis discussion.

Despite the generality and simplicity of scenario-based methods, they have several important drawbacks. They require sampling from the posterior, with a sample complexity that scales poorly with the dimension $d$, number of actions $k$, and particularly confidence level $1-\delta$. They do not provide theoretical guarantees for the regret of the obtained policy and offer no insights into how the regret scales with the parameters of the problem.

Finally, minimizing the regret in~\eqref{eq:regret-high-confidence-VaR} over the space of {\em randomized} policies is challenging using scenario-based methods because it requires solving a mixed-integer linear program~\cite{Lobo2020}. Other ideas have been explored~\cite{Calafiore2005, Brown2020} but a detailed study of such algorithms is beyond our scope.


\section{Minimizing Analytical Regret Bounds}
\label{sec:algo}

In this section, we propose our new approach for minimizing the Bayesian regret, $\mathfrak R_{\delta}(\bm{\pi})$, defined in~\eqref{eq:regret-high-confidence-VaR}. In particular, we derive two upper bounds on $\mathfrak R_{\delta}(\bm{\pi})$ that complement each other depending on the relative sizes of the feature vector $d$ and action space $k$. We also prove a lower bound on $\mathfrak R_{\delta}(\bm{\pi})$ that shows our upper bound is tight. Finally, we propose our \texttt{BRMOB} algorithm that aims at jointly minimizing our two upper bounds. The proofs of this section are in \cref{sec:algo-proofs}.


\subsection{Bayesian Regret Bounds}
\label{subsec:Bayes-regret-bounds}

To avoid unnecessary complexity, we assume in this section that the posterior distribution over the reward parameter is {\em Gaussian}. We show analogous results for the general {\em sub-Gaussian} case in \cref{sec:sub-gauss-post}. 

\begin{assumption} 
\label{asm:gaussian-posterior}
The posterior over the latent reward parameter is distributed as $\bm{\tilde{\theta}}_D \sim \mathcal{N}(\bm{\mu}, \bm{\Sigma})$, with mean $\bm{\mu}\in \Real^d$ and a \emph{positive definite} covariance matrix $\bm{\Sigma}\in \Real^{d\times d}$. 
\end{assumption}

We begin by showing that under \cref{asm:gaussian-posterior}, the regret $r(a; \bm{\tilde{\theta}}_D) - r(\bm{\pi}; \bm{\tilde{\theta}}_D)$ of any policy $\bm{\pi}\in \probs{k}$ with respect to $a \in \mathcal{A}$ has a Gaussian distribution.

\begin{lemma} \label{lem:regret-normal}
Suppose that $\bm{\tilde{\theta}}_D \sim \mathcal{N}(\bm{\mu}, \bm{\Sigma})$. Then, for any policy $\bm{\pi}\in \probs{k}$, the Bayesian regret in~\eqref{eq:regret-high-confidence-VaR} can be written as
\begin{equation}
\label{eq:var-maximum}
\mathfrak{R}_{\delta}(\bm{\pi}) \;=\; \var{1-\delta}{ \max_{a\in \mathcal{A}} \; \tilde{x}_a^{\bm{\pi}}}\,,
\end{equation}
where $\;\tilde{x}_a^{\bm{\pi}}\sim\mathcal N(\mu_a^{\bm{\pi}},\sigma_a^{\bm{\pi}})\;$ with 
\begin{equation} \label{eq:policy-action-regret}
\mu_a^{\bm{\pi}} = \bm{\mu}\tr \bm{\Phi} (\bm{1}_a - \bm{\pi}), \qquad
\sigma_a^{\bm{\pi}} = \norm{ \bm{\Phi} (\bm{1}_a - \bm{\pi})}_{\bm{\Sigma}}\;.
\end{equation}
\end{lemma}
\Cref{lem:regret-normal} points to the main challenge in deriving tight bounds on $\mathfrak{R}_{\delta}(\bm{\pi})$. Even when $\bm{\tilde{\theta}}_D$ is normally distributed, the random variable $\max_{a\in \mathcal{A}} \; \tilde{x}_a^{\bm{\pi}}$ is unlikely to be Gaussian. The lack of normality prevents us from deriving an \emph{exact} analytical expression for $\mathfrak R_{\delta}(\bm{\pi})$ using~\eqref{eq:var-normal}. In the remainder of the section, we derive two separate techniques for upper bounding the VaR of the maximum of random variables in~\eqref{eq:var-maximum}, thereby also bounding the Bayesian regret $\mathfrak R_{\delta}(\bm{\pi})$.

Our first bound expresses the overall regret as a maximum over individual action regrets. We refer to it as an {\em action-set bound}, because it grows with the size of the action space $k$, and state it in \cref{thm:var-max-combined}.
\begin{theorem} 
\label{thm:var-max-combined}
The regret for any policy $\bm{\pi} \in \probs{k}$ satisfies 
\begin{subequations}\label{eq:var-union-normal}
\begin{align}
\label{eq:var-union-normal-one}  
 \mathfrak{R}_{\delta}(\bm{\pi})
&\le \min_{\bm{\xi}\in \probs{k}} \max_{a\in \mathcal{A}} \, \mu^{\bm{\pi}}_a + \sigma_a^{\bm{\pi}} \cdot z_{1-\delta \xi_a} \\
\label{eq:var-union-normal-two}  
&\le \min_{\bm{\xi}\in \probs{k}} \max_{a\in \mathcal{A}} \, \mu^{\bm{\pi}}_a + \sigma_a^{\bm{\pi}} \cdot \sqrt{2 \log(\nicefrac{1}{\delta\xi_a}) }
\;, 
\end{align}
\end{subequations}
where $z_{1-\delta\xi_a}$ is the $(1-\delta\xi_a)$-th standard normal quantile.
\end{theorem}
%

A special case of~\eqref{eq:var-union-normal} is when $\bm{\xi} = 1/k \cdot \bm{1}$ is uniform, in which case is simplifies to  
\begin{equation} 
\label{eq:var-union-normal-simple}
\mathfrak{R}_{\delta}(\bm{\pi})
\; \le\; 
\max_{a\in \mathcal{A}} \, \mu^{\bm{\pi}}_a + \sigma_a^{\bm{\pi}}
\cdot \sqrt{2 \log(k/\delta)} \;.
\end{equation}
This shows that the action-set bound in \cref{thm:var-max-combined} grows sub-logarithmically with the number of actions $k$. 


Because the bound in \cref{thm:var-max-combined} is based on a union bound, the question of its tightness is particularly salient. To address this, we prove a lower bound on the regret when the arms are independent (e.g.,~multi-armed bandits).

\begin{theorem} \label{thm:lower bound}
Suppose that $\bm{\pi} \in \probs{k}$ is a \emph{deterministic} policy such that $\pi_{a_1} = 1$ for $a_1 \in \mathcal{A}$ without loss of generality. When $\mu_{2} = \mu_{3} = \dots = \mu_k$, $\bm{\Sigma}$ is diagonal with $\Sigma_{2,2} = \Sigma_{3,3} = \dots \Sigma_{k,k}$, and $\bm{\Phi} = \bm{I}$, then
  \[
    \mathfrak{R}_{\delta}(\bm{\pi})
    \; \ge\;
    \mu^{\bm{\pi}}_{a_2} + \sigma_{a_2}^{\bm{\pi}} \cdot \kappa_{\mathrm{l}}(k-1),
  \]
where
\[
\kappa_{\mathrm{l}}(k) = -1 + \sqrt{1 - \log (\sqrt{2\pi}) - 2\log \left(1-(1-\delta)^{\nicefrac{1}{k}}\right) }.
\]
\end{theorem}

\begin{figure}
  \centering
  \includegraphics[width=0.4\linewidth]{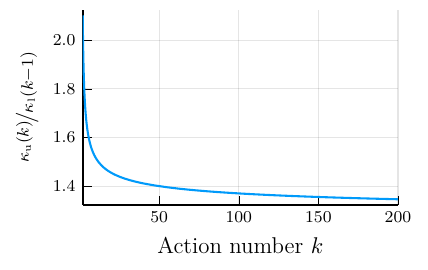}
  \caption{The quotient of the upper bound coefficient $\kappa_{\mathrm{u}}(k)$ and the lower bound coefficient $\kappa_{\mathrm{l}}(k)$.}
  \label{fig:kappa-comparison}
\end{figure}

The lower bound in \cref{thm:lower bound} indicates that \cref{thm:var-max-combined} is tight. For an ease of reference we use $\kappa_{\mathrm{u}}(k) = \sqrt{2 \log(k/\delta)}$ to refer to the coefficient on the RHS of~\eqref{eq:var-union-normal-simple}. The main difference between the upper and lower bounds are the coefficients $\kappa_{\mathrm{u}}(k)$ and $\kappa_{\mathrm{l}}(k-1)$.  One can readily show that $\kappa_{\mathrm{u}}(k) \in O(\kappa_{\mathrm{l}}(k-1))$, since $\nicefrac{\kappa_{\mathrm{u}}(1)}{\kappa_{\mathrm{l}}(1)} < 10$ when $\delta \le \nicefrac{1}{2}$ and $\nicefrac{\kappa_{\mathrm{u}}(k)}{\kappa_{\mathrm{l}}(k-1)}$ is a non-increasing function of $k$. \Cref{fig:kappa-comparison} depicts the quotient of the upper and lower bound coefficients as a function of $k$ for $\delta=0.1$.

We now state our second upper bound on the regret in \cref{thm:var-max-l2}. We refer to it as the {\em parameter-space bound} because it grows with the dimension $d$ of the parameter. 

\begin{theorem}
\label{thm:var-max-l2}
The regret for any policy $\bm{\pi} \in \probs{k}$ satisfies
\begin{subequations} \label{eq:var-max-l2-eq}
\begin{align}
\label{eq:var-max-l2-eq-one}
\mathfrak{R}_{\delta}(\bm{\pi})
&\le \max_{a\in \mathcal{A}} \, \mu^{\bm{\pi}}_a+ \sigma^{\bm{\pi}}_a \cdot \sqrt{\chi^2_d(1-\delta)} \\
\label{eq:var-max-l2-eq-two}
&\le \max_{a\in \mathcal{A}} \, \mu^{\bm{\pi}}_a+ \sigma^{\bm{\pi}}_a \cdot 5d \log(1/\delta) 
\,, 
\end{align}
\end{subequations}
where $\chi^2_d(1-\delta)$ is the $(1-\delta)$-th quantile of the $\chi^2$ distribution with $d$ degrees of freedom.
\end{theorem}

Note that a growing body of literature argues that using credible regions in constructing robust approximations of VaR is overly conservative when used with linear or concave functions~\cite{Gupta2019, Bertsimas2021, Petrik2019}. However, these results do not apply to our setting because the maximum in~\eqref{eq:var-maximum} is \emph{non-concave}. 

To compare our two upper bounds in~\eqref{eq:var-union-normal} and~\eqref{eq:var-max-l2-eq}, it is sufficient to compare the terms $z_{1-\delta \xi_a}$ and $\sqrt{\chi^2_d(1-\delta)}$. From \cref{thm:var-max-combined,thm:var-max-l2}, we can conclude that the second upper bound is preferable when $d < \log k$.


\subsection{Optimization Algorithm}
\label{subsec:opt-algo}

\begin{algorithm*}[ht!]
\caption{\texttt{BRMOB}: Bayesian Regret Minimization for Offline Bandits}
\label{alg:BRMOB}
\KwIn{Posterior parameters $\bm{\mu}$ and $\bm{\Sigma}$, $\;\;\;$ risk tolerance $\delta\in (0,\nicefrac{1}{2})$, $\;\;\;$ feature matrix $\bm{\Phi}\in\mathbb R^{d\times k}$, $\;\;\;$ \# of iterations $m$} 
Initialize $\;\nu^0_a \gets \min \left\{  \sqrt{\chi^2_d(1-\delta)}, z_{1-\nicefrac{\delta}{k}} \right\},\;\forall a \in \mathcal{A}\;$; $\quad$ $\;\; i \gets 0\;$\;
Minimize regret bounds: Let $\rho^i$ and $\bm{\pi}^i$ be the optimizers of
\label{line:min-regret}
\begin{equation} 
\label{eq:main-optimization}
\operatorname*{minimize}_{\bm{\pi}, \; \bm{s}\in \Real^k_+, \; \rho\in \Real}
\rho \quad \operatorname{subject\;to} \quad \bm{1}\tr \bm{\pi} = 1, \quad 
\rho \ge \mu^{\bm{\pi}}_a + s_a\cdot \nu_a^i,\quad  s_a^2 \ge (\sigma^{\bm{\pi}}_a)^2,\,\; \forall a\in \mathcal{A} .
\end{equation} \\
\For{$\;i = 1, \ldots , m\;$}{
Tighten regret bounds: Let $\bm{\xi}^i$ be an optimizer of
\begin{equation}
\label{eq:finetuning}
\operatorname*{minimize}_{\bm{\xi}, \, \bm{s}\in \Real^k_+, \, \bm{l} \in \Real^k, \; \rho\in \Real} \rho \;\;\; \operatorname{subject\;to} \;\;\; \bm{1}\tr \bm{\xi} = \delta, \;\;\; \rho \ge \mu^{\bm{\pi}^{i-1}}_a \!\!+ \sigma^{\bm{\pi}^{i-1}}_a \cdot s_a, \;\;\; s_a^2 \ge -2 l_a, \;\;\; l_a \le \log \xi_a, \; \forall a\in \mathcal{A} .
\end{equation}\\
Set $\nu_a^i \gets  z_{1-\delta\xi^i_a}, \; \forall a\in \mathcal{A}$ \;
Solve~\eqref{eq:main-optimization} and let $\rho^i$ and $\bm{\pi}^i$ be its optimizers \;
}
\label{line:best-policy}
$i\opt \gets \arg\min_{i = 0, \dots, m} \; \rho^i$  \label{Line-Opt-Regret}
\tcp*{Choose policy with the best regret guarantee}
\Return{ {\em randomized policy} $\bm{\pi}^{i\opt}$, {\em regret upper bound} $\rho^{i\opt}$} \;
\end{algorithm*}

We now describe our main algorithm, \emph{Bayesian Regret Minimization for Offline Bandits}~(\texttt{BRMOB}), whose pseudo-code is reported in \cref{alg:BRMOB}. Before describing \texttt{BRMOB} in greater detail, it is important to note that it returns a \emph{randomized} policy. Unlike in online bandits, here the goal of randomization among the actions is not to explore, but rather to reduce the risk of incurring high regret. The numerical results in \cref{sec:numer-demonstr} show that the ability to randomize over actions significantly reduces Bayesian regret in many situations. 

\texttt{BRMOB}'s strategy is to compute a policy with the \emph{minimum} regret guarantee. In \cref{line:min-regret}, it computes a policy $\bm{\pi}^0$ that simultaneously minimizes our two proposed upper bounds: the one in \cref{thm:var-max-combined} with a uniform $\bm{\xi}$ as given in~\eqref{eq:var-union-normal-simple}, and the one in \cref{thm:var-max-l2} as given in~\eqref{eq:var-max-l2-eq}. The bounds can be optimized jointly because they differ only in constant $\nu$. The optimization in~\eqref{eq:main-optimization} is a second-order conic program (SOCP), because $\bm{\nu}\ge \bm{0}$ and can be solved very efficiently~\cite{mosek2022, Lubin2023}. The actual time complexity depends on the particular SOCP solver used, but most interior-point algorithms run in $O(k^6)$ complexity or faster~\cite{Kitahara2018}.

After completing \cref{line:min-regret}, \texttt{BRMOB} proceeds with $m$ iterations of tightening the regret bound and improving the policy. In each iteration $i$, it tightens the regret bound in \cref{thm:var-max-combined} by optimizing $\bm{\xi}^i$ in~\eqref{eq:finetuning} for the incumbent policy $\bm{\pi}^{i-1}$. The minimum in~\eqref{eq:finetuning} can be computed efficiently using exponential and second-order cones~\cite{mosek2022, Lubin2023}. Exponential conic optimization is hypothesized to be polynomial time, but this fact has not been established yet to the best of our knowledge. The algorithm then minimizes the tightened bound by solving~\eqref{eq:main-optimization} and obtains an improved policy $\bm{\pi}^i$.

The tightening steps in \cref{alg:BRMOB} can be seen as a coordinate descent procedure for joint minimization of $\bm{\pi}$ and $\bm{\xi}$ in~\eqref{eq:var-union-normal}. It would be preferable to minimize the bound simultaneously over $\bm{\pi}$ and $\bm{\xi}$, but such optimization appears to be intractable.

Finally, \cref{alg:BRMOB} returns a policy in the set $\{\bm{\pi}^i\}_{i=0}^m$ with the smallest regret bound $\rho^i$ in \cref{line:best-policy}. Although $\rho^i$ will be generally non-increasing with an increasing $i$, this is not guaranteed. This is because the tightening step in~\eqref{eq:finetuning} minimizes the bounds in~\eqref{eq:var-union-normal-two} and~\eqref{eq:var-max-l2-eq-two}. These bounds are generally looser than the bounds in~\eqref{eq:var-union-normal-one} and~\eqref{eq:var-max-l2-eq-one} optimized by~\eqref{eq:main-optimization}.

We provide a worst-case error bound on the regret of \texttt{BRMOB} in \cref{sec:bayes-regr-analys}. Our regret bound holds for any number of tightening steps, including $m=0$. We focus on bounds that are independent of $m$ for the sake of simplicity, since the improvements that arise from the tightening procedure can be difficult to quantify cleanly.

We conclude this section with the following result that shows \texttt{BRMOB} indeed minimizes the regret upper bounds in \cref{thm:var-max-combined,thm:var-max-l2} as intended.
\begin{proposition} \label{cor:algo-is-great}
Suppose that \texttt{BRMOB} returns a policy $\bm{\hat{\pi}}\in \probs{k}$ and a regret bound $\hat{\rho}$. Then
\begin{equation}
  \mathfrak{R}_{\delta}(\bm{\hat\pi})
  \; \le \; 
 \hat\rho
  \; \le\; 
  \min_{\bm{\pi}\in \probs{k}} \max_{a\in \mathcal{A}} \, \mu_a^{\bm{\pi}}(n) +  \sigma^{\bm{\pi}}_a(n) \cdot  \eta \;,
\end{equation}
where $\;\eta = \min \left\{\sqrt{2 \log (k/\delta) } , \sqrt{5 d \log(1/\delta) } \right\}\;$. 
\end{proposition}


\section{Regret Analysis}
\label{sec:bayes-regr-analys}

In this section, we derive a regret bound for \texttt{BRMOB} and compare it with that of \texttt{FlatOPO}~\cite{Hong2023}, an LCB-style algorithm. We use a frequentist analysis to bound the Bayesian regret of \texttt{BRMOB} as a function of $k$, $d$, number of samples $n$, and coverage of the dataset $D$. \Cref{sec:comparison-with-lcb} concludes by arguing that the general LCB approach in~\eqref{eq:lcb-hong} is unsuitable for minimizing Bayesian regret; see \cref{sec:other-objectives} for other objectives that can be optimized using LCB-style algorithms. Our lower bound shows that LCB can match \texttt{BRMOB}'s regret only if the confidence penalty $\beta$ is very small and \emph{decreases} with $k$ and $d$. All the proofs of this section are reported in \Cref{sec:theory-proofs}.


\subsection{Sample-Based Regret Bound} 
\label{sec:sample-based-regret}

As in prior work~\cite{Hong2023}, we assume a Gaussian prior distribution over the reward parameter $P_{\bm{\tilde{\theta}}} = \mathcal{N}(\bm{\mu}_0, \bm{\Sigma}_0)$ with an invertible $\bm{\Sigma}_0$, and Gaussian rewards $\tilde{y} \sim \mathcal{N}\big(r(a;\bm{\tilde\theta})=\bm{\phi}_a\tr \bm{\tilde\theta}, \bar{\sigma}^2\big)$ for each action $a\in\mathcal A$. As a result, the posterior distribution over the parameter given a dataset $D = \{(a_i,y_i)\}_{i=1}^n$ is also Gaussian $\bm{\tilde{\theta}}_D \sim \mathcal{N}(\bm{\mu}_n, \bm{\Sigma}_n)$ with 
\begin{equation}
\label{eq:guassian-posterior}
\begin{aligned} 
\bm{\Sigma}_n &= ( \bm{\Sigma}_0^{-1} + \bar{\sigma}^{-2} \bm{G}_n )^{-1}, \\
\bm{\mu}_n &= \bm{\Sigma}_n (\bm{\Sigma}_0^{-1}\bm{\mu}_0 + \bar{\sigma}^{-2} \bm{B}_n \bm{y}_n),
\end{aligned}
\end{equation}
and where  $\bm{B}_n = (\phi_{a_i})_{i=1}^n$ is the matrix with observed features in its columns, $\bm{y}_n = (y_i)_{i=1}^n$ is the vector of observed rewards, and $\bm{G}_n = \bm{B}_n\tr \bm{B}_n$ is the empirical covariance matrix (see Bayesian linear regression for example in \citealt{Rasmussen2006, Deisenroth2021}).

To express the regret bound as a function of the dataset $D$, we make the following standard quality assumption. 
\begin{assumption} \label{asm:good-data}
The feature vectors satisfy $\|\bm{\phi}_a\|_2 \le 1, \forall a\in \mathcal{A}$, and there exists a $\gamma > 0$ such that
\begin{equation*}
    \bm{G}_n \; \succeq\;  \gamma n \cdot \bm{\phi}_a \bm{\phi}_a\tr , \qquad \forall a\in \mathcal{A}, \;\forall n \ge 1\,.
\end{equation*}
\end{assumption}

Intuitively, \cref{asm:good-data} states that the dataset provides sufficient information such that the norm of the covariance matrix $\bm{\Sigma}_n$ of the posterior distribution over $\bm{\tilde{\theta}}_D$ decreases linearly with $n$. From a frequentist perspective, this assumption holds with high probability by the Bernstein-Von-Mises theorem under mild conditions~\cite{Vaart2000}.

We are now ready to bound the Bayesian regret of \texttt{BRMOB}. We state the bound for the general case and then tighten it when $\bm{\mu}_n = \bm{0}$ (only the variance of actions matters).

\begin{theorem}
\label{thm:regret-bound}
Suppose that the parameter has a Gaussian posterior $\bm{\tilde{\theta}}_D \sim \mathcal{N}(\bm{\mu}_n, \bm{\Sigma}_n)$ and \texttt{BRMOB} returns a policy $\bm{\hat \pi}$. Then, the regret of \texttt{BRMOB} is bounded as $\mathfrak{R}_{\delta}(\bm{\hat \pi})
\le 2\eta$, where
\begin{equation} \label{eq:regret-bound}
   \eta =  \sqrt{ \frac{\min \left\{ 2 \log (\nicefrac{k}{\delta}), 5 d \log(\nicefrac{1}{\delta})  \right\} }
{\lambda_{\max}(\bm{\Sigma}_0)^{-1}  + \gamma n\bar{\sigma}^{-2}} }\;.
\end{equation}
Moreover, if $\bm{\mu}_n = 0$ then $ \mathfrak{R}_{\delta}(\bm{\hat{\pi}}) \le 2 \left(1- \max_{a'\in \mathcal{A}} \hat{\pi}_{a'}\right) \eta$ with $\max_{a'\in \mathcal{A}} \hat{\pi}_{a'} \ge \nicefrac{1}{d+1}\;$.
\end{theorem}

\subsection{Comparison with FlatOPO}
\label{sec:comp-flat-opo}

We now compare the regret bound of \texttt{BRMOB} with that of \texttt{FlatOPO}~\cite{Hong2023}, an LCB-based algorithm for regret minimization in Bayesian offline bandits. As discussed in \cref{sec:baseline-algorithms}, using the LCB principle is the most common approach to regret minimization in offline decision-making. \citet{Hong2023} derived the following regret bound for \texttt{FlatOPO} under \cref{asm:good-data}: 
\begin{equation}  \label{eq:flatopo-regret}
  \mathfrak{R}_{\delta}(\bm{\hat\pi}) \;\le\;
  2\sqrt{ \frac{5 d^2 \log(\nicefrac{1}{\delta})}
    {\lambda_{\max}(\bm{\Sigma}_0)^{-1} + \gamma n\bar{\sigma}^{-2}} }\;,
\end{equation}
where $\bm{\hat\pi}$ is a \emph{deterministic} policy returned by the algorithm. 

Comparing our regret bound for \texttt{BRMOB} in~\eqref{eq:regret-bound} with \texttt{FlatOPO}'s in~\eqref{eq:flatopo-regret}, we notice two main improvements. The first one is that the \texttt{BRMOB}'s regret is bounded by $\sqrt{\log k}$. Thus, when the number of actions $k$ satisfies $k \ll \exp(d)$, the regret guarantee of \texttt{BRMOB} can be dramatically lower than that achieved by \texttt{FlatOPO}. It is unclear how one could extend the existing analysis in \citet{Hong2023} to bound its regret in terms of $k$. Its design and analysis rely on a robust set, which is difficult to restrict using $k$.

The second improvement is that the regret bound of \texttt{BRMOB} grows $\sqrt{d}$ slower than \texttt{FlatOPO}'s, which is a significant reduction in regret. This improvement is probably a consequence of our tighter analysis rather than better algorithmic design. The analysis in \citet{Hong2023} uses a general upper bound on the trace of a rank one matrix, which introduces an unnecessary $\sqrt{d}$ term. Yet, applying our techniques to \texttt{FlatOPO} yields additional constant terms missing in~\eqref{eq:flatopo-regret}.


\subsection{Limitation of LCB}
\label{sec:comparison-with-lcb}

In this section, we argue that the popular LCB approach is inherently unsuitable for minimizing Bayesian \emph{regret} in offline bandits. As we discussed in \cref{sec:comp-flat-opo}, \texttt{BRMOB} achieves significantly better regret guarantees than \texttt{FlatOPO}. Our numerical results in \cref{sec:numer-demonstr} also show that \texttt{BRMOB} outperforms \texttt{FLatOPO}. However, these results are obtained for a particular value of $\beta$ in~\eqref{eq:lcb-hong}. Our theoretical analysis suggests that even a simple \texttt{Greedy} algorithm, which uses $\beta = 0$ in~\eqref{eq:lcb-hong}, can significantly outperform LCB. The intuition behind the LCB approach is that one should prefer actions with low uncertainty, and thus, limited downside. This intuition is correct when the goal is to maximize the VaR of \emph{reward} as shown in \cref{sec:optim-bayes-regr}. However, this intuition does not apply when the objective is \emph{regret} minimization. In fact, actions with low uncertainty also have limited upside and high regret, and thus, as we show, penalizing high variance actions is counterproductive. 

We now construct a simple class of offline bandit problems to illustrate LCB's limitations. For this class of problems, we show that the \emph{lower bound} on the regret of LCB can be far greater than the upper bound on \texttt{BRMOB}, or even \texttt{Greedy}, policy. In what follows, we assume that LCB computes the high-confidence lower bound as in~\eqref{eq:lcb-hong} for some value of $\beta$. 

\begin{example} \label{exm:lcb-counterexample}
Consider a class of offline bandit problems parametrized with the $\beta$ used in~\eqref{eq:lcb-hong}. The bandit has $k \ge 2$ arms, feature dimension $d=k$, and a feature matrix $\bm{\Phi} = \bm{I}$. Suppose that the posterior covariance over the reward parameter $\bm{\Sigma}\in \Real^{k \times k}$ is diagonal with the diagonal elements $\sigma_1 = 0$ and $\sigma_2 = \dots = \sigma_k$, and the posterior mean has the following form: $\mu_1 = 0$ and $\mu_2 = \dots =\mu_k = \beta \cdot  \sigma_2 \ge 0$. 
\end{example}

The intuition underlying the bandit problems in \cref{exm:lcb-counterexample} is as follows. It has one action, $a_1$, with low reward and low variance. The other $k-1$ arms are i.i.d.~with higher mean and variance. LCB prefers to take action $a_1$ because of its low variance and forgoing the higher mean of the other actions. The next theorem shows that taking any of the other actions with a higher mean, as would be chosen by \texttt{BRMOB}, or even \texttt{Greedy} that selects an action with the largest posterior mean, leads to a far lower regret.

\begin{theorem} \label{thm:counterexample}
Consider the bandit problems in \cref{exm:lcb-counterexample} and assume a realization of LCB with a coefficient $\beta > 0$ that breaks ties by choosing an $a_i$ with the smallest $i$. Then, LCB returns $\bm{\pi}_{\mathrm{LCB}}\in \probs{k}$ with $\bm{\pi}_{\mathrm{LCB}}(a_1) = 1$ and 
  \begin{equation} \label{eq:lcb-lower bound}
    \mathfrak{R}_{\delta}(\bm{\pi}_{\mathrm{LCB}})
    \; \ge \;
    (\beta + \kappa_{\mathrm{l}}(k)) \cdot \sigma_{a_2}.
\end{equation}
Moreover, \texttt{Greedy} with the same tie-breaks will return a policy $\bm{\pi}_{\mathrm{G}} \in \probs{k}$ with $\bm{\pi}_{\mathrm{G}}(a_2) = 1$ and
  \begin{equation} \label{eq:greedy-upper bound}
    \mathfrak{R}_{\delta}(\bm{\pi}_{\mathrm{G}})
    \; \le \;
    \sqrt{2} \cdot  \sigma_{a_2} \cdot \kappa_{\mathrm{u}}(k) .
\end{equation}
Finally, \texttt{BRMOB}'s regret also satisfies the bound in~\eqref{eq:greedy-upper bound}.
\end{theorem}

\begin{figure}
  \centering
  \includegraphics[width=0.4\linewidth]{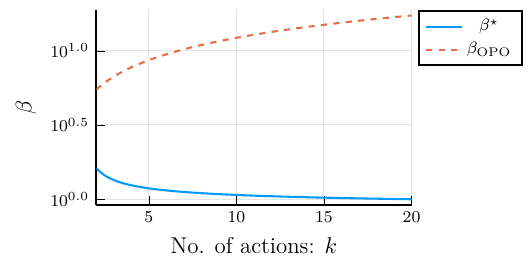}
  \caption{The value of $\beta$ used by \texttt{FlatOPO} in \cref{exm:lcb-counterexample}, $\beta_{\mathrm{OPO}}$, and the upper bound $\beta\opt$ that may avoid the under-performance of LCB, defined in~\eqref{eq:beta-values}, as functions of the number of actions $k$. 
  }
  \label{fig:beta-comparison}
\end{figure}

\Cref{thm:counterexample} shows that even in a simple class of problems, \texttt{Greedy} (or \texttt{BRMOB}) computes a policy that outperforms LCB significantly. The increase in regret of LCB versus \texttt{Greedy} (or \texttt{BRMOB}) can be bounded from below as
\begin{equation}
\label{eq:diff-regret}
\!\mathfrak{R}_{\delta}(\bm{\pi}_{\mathrm{LCB}}) - \mathfrak{R}_{\delta}(\bm{\pi}_{\mathrm{G}}) \! \ge \! \left(\beta \!+\! \kappa_{\mathrm{l}}(k) \!-\! \sqrt{2} \kappa_{\mathrm{u}}(k) \right) \sigma_{a_2}.
\end{equation}
Note that the bound, when positive, can be made arbitrarily large by scaling $\sigma_{a_2}$.

\begin{figure*}
  \centering
  \includegraphics[width=0.3\linewidth]{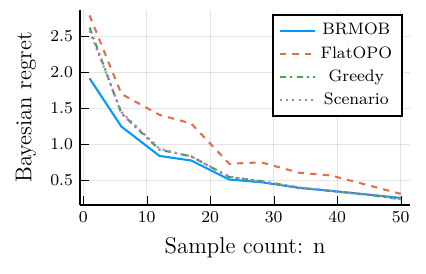}
  \hspace{0.02\linewidth}
  \includegraphics[width=0.3\linewidth]{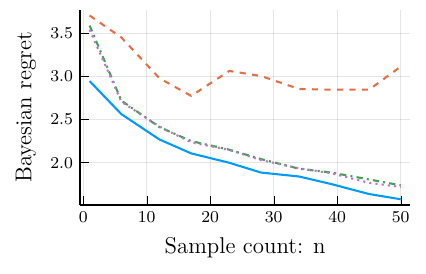}
  \hspace{0.02\linewidth}
  \includegraphics[width=0.3\linewidth]{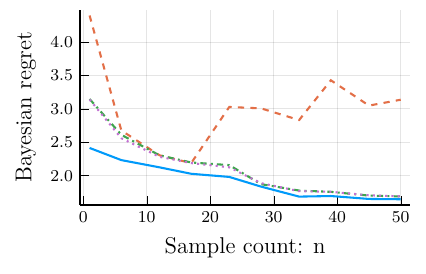}
  \caption{ Bayesian regret with $k=d=5$ {\em (left)}, $k=d=50$ {\em (middle and right)}. The prior mean is $\bm{\mu}_0 = \bm{0}$ {\em (left and middle)} and $(\bm{\mu}_0)_a = \sqrt{a}\;$ for $a = 1, \dots  50$ \emph{(right)}.}
\label{fig:independent-actions}
\end{figure*}

\begin{figure*}
  \centering
  \includegraphics[width=0.3\linewidth]{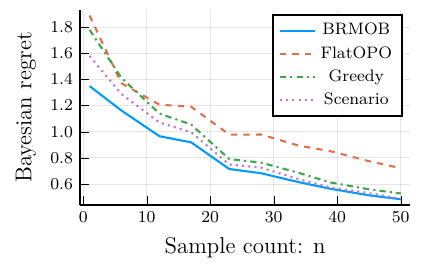}
  \hspace{0.02\linewidth}
  \includegraphics[width=0.3\linewidth]{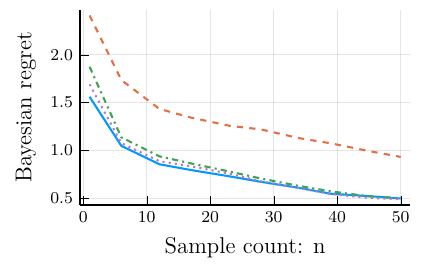}
  \hspace{0.02\linewidth}
  \includegraphics[width=0.3\linewidth]{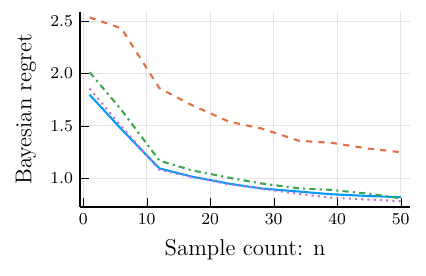}
  \caption{ Bayesian regret with $d=4$ and $k=10$ {\em (left)}, $k=50$ {\em (middle)}, and $k=100$ {\em (right)}.}
\label{fig:dependent-actions}
\end{figure*}

Using algebraic manipulation of the bound in~\eqref{eq:diff-regret}, we can show that $\beta$ should satisfy the following condition for LCB to perform better than \texttt{Greedy} and \texttt{BRMOB}: 
\begin{equation}
\label{eq:beta-values}
\beta \leq \beta\opt = \frac{\sqrt{2} \kappa_{\mathrm{u}}(k)}{\kappa_{\mathrm{l}}(k-1)} - 1.    
\end{equation}
The inequality in~\eqref{eq:beta-values} indicates that penalizing an action's uncertainty with a $\beta$ greater than $\beta\opt$ \emph{increases} the regret of LCB. 
For comparison, the $\beta$ used by \texttt{FlatOPO} for the class of problems in \cref{exm:lcb-counterexample} in which $d=k$ is $\beta_{\mathrm{OPO}} = \sqrt{5 k \log(\nicefrac{1}{\delta})}$. \Cref{fig:beta-comparison} shows that for $\delta = 0.1$, $\beta_{\mathrm{OPO}}$ exceeds $\beta\opt$, and thus, violates the condition in~\eqref{eq:beta-values} and performs worse than \texttt{Greedy} and \texttt{BRMOB} for all values of $k$. 


\section{Numerical Results}
\label{sec:numer-demonstr}

In this section, we compare \texttt{BRMOB} to several baselines on synthetic domains. Here, we evaluate the basic version of \texttt{BRMOB} with $m=0$ iterations and defer results that demonstrate the improvement from the tightening step to \cref{sec:addit-exper-results}. Particularly, we compare it to (i) \texttt{FlatOPO}~\cite{Hong2023} that is based on the LCB principle, (ii) \texttt{Greedy} method which selects an action $a$ with the largest value of $\mu_a$, and finally, (iii) \texttt{Scenario}, the scenario-based method described in~\eqref{eq:scenario-based} in \cref{sec:baseline-algorithms}. We execute \texttt{Scenario} with 4000 samples from the posterior. Increasing this number did not improve our results.

Our experiments use synthetic domains, each defined by a normal prior $(\bm{\mu}_0, \bm{I})$ and a feature matrix $\bm{\Phi}$. To evaluate the Bayesian regret as a function of data size $n$, we first sample a single large dataset and then vary the number of data points $n$ from this set used to estimate the posterior distributions. The regret for each policy is computed by a scenario-based algorithm that samples from the posterior and computes the empirical VaR. We use the error tolerance of $\delta = 0.1$ throughout. Results are averaged over $100$ runs of this process to reduce variance. As confidence intervals were negligible for all algorithms except \texttt{FlatOPO}, to avoid clutter, we do not plot them here and refer the reader to \cref{sec:addit-exper-results} for additional details. 

We evaluate the algorithms on three domains. The first one uses $k=d$ actions, an identity feature matrix $\bm{\Phi} = \bm{I}$, and zero prior mean $\bm{\mu}_0 = \bm{0}$. The second one is the same, except $(\bm{\mu}_0)_a = \sqrt{a}$ to simulate varying rewards for actions. Finally, the third one fixes dimension $d = 4$ and varies $k$ while using randomly generated features from the $\ell_{\infty}$-ball.

\Cref{fig:independent-actions,fig:dependent-actions} summarize our numerical results. They consistently show across all domains that \texttt{BRMOB} significantly outperforms all the other algorithms, particularly when the posterior uncertainty is large. The only challenging setting for \texttt{BRMOB} is when $k \gg d$. Note that \texttt{FlatOPO}'s performance is noisy in \Cref{fig:independent-actions} because its $\beta$ coefficient grows fast with $d$. A common practice is to tune $\beta$, but we did not find any value of $\beta$ for which \texttt{FlatOPO} performs better than \texttt{Greedy}, which is consistent with our theoretical analysis in \cref{sec:comparison-with-lcb}. It is also notable that \texttt{Greedy} outperforms LCB significantly, furnishing further evidence that  LCB is unsuitable for minimizing regret. We provide additional results, including confidence bounds and runtime, in \cref{sec:addit-exper-results}.


\section{Conclusion}
\label{sec:disc-conc}

We proposed \texttt{BRMOB}, a new approach for Bayesian regret minimization in offline bandits, that is based on jointly minimizing two analytical upper bounds on the Bayesian regret. We proved a regret bound for \texttt{BRMOB} and showed that it compares favorably with an existing LCB-style algorithm \texttt{FlatOPO}~\cite{Hong2023}. Finally, we showed theoretically and empirically that the popular LCB approach is unsuitable for minimizing Bayesian regret.

Our approach can be extended to several more general settings. The algorithm and bounds can generalize to sub-Gaussian posterior distributions as described in \cref{sec:sub-gauss-post}. The algorithm can also be extended to {\em contextual} linear bandits by computing a separate policy $\pi$ for each context individually or by assuming a random context. Another important future direction is understanding the implications of our results to frequentist settings where similar concerns about the value of the LCB approach have been raised~\cite{Xie2022c, Xiao2021b}. 




\subsection*{Acknowledgments}

We thank the anonymous referees for their helpful comments and suggestions. The work in the paper was supported, in part, by NSF grants 2144601 and 2218063. In addition, this work was performed in part while Mohammad Ghavamzadeh and Marek Petrik were at Google Research.

\subsection*{Impact Statement}

This paper aims to advance the theory field of Machine Learning. There are many potential societal consequences of our work, but it is difficult to speculate what they may be. Of course, we sincerely hope that the impacts will only be positive. 

\bibliographystyle{icml2024}
\bibliography{offlinerl}

\begin{thebibliography}{46}
\providecommand{\natexlab}[1]{#1}
\providecommand{\url}[1]{\texttt{#1}}
\expandafter\ifx\csname urlstyle\endcsname\relax
  \providecommand{\doi}[1]{doi: #1}\else
  \providecommand{\doi}{doi: \begingroup \urlstyle{rm}\Url}\fi

\bibitem[{Ahmadi-Javid}(2012)]{Ahmadi-Javid2012}
{Ahmadi-Javid}, A.
\newblock Entropic {{Value-at-Risk}}: A new coherent risk measure.
\newblock \emph{Journal of Optimization Theory and Applications}, 155\penalty0
  (3):\penalty0 1105--1123, 2012.

\bibitem[Angelotti et~al.(2021)Angelotti, Drougard, and Chanel]{Angelotti2021}
Angelotti, G., Drougard, N., and Chanel, C. P.~C.
\newblock Exploitation vs caution: Risk-sensitive policies for offline
  learning.
\newblock \emph{arXiv:2105.13431 [cs, eess]}, 2021.

\bibitem[ApS(2022)]{mosek2022}
ApS, M.
\newblock \emph{The MOSEK optimization toolbox for MATLAB manual. Version
  10.0.}, 2022.

\bibitem[Behzadian et~al.(2021)Behzadian, Russel, Ho, and
  Petrik]{Behzadian2021}
Behzadian, B., Russel, R., Ho, C.~P., and Petrik, M.
\newblock Optimizing percentile criterion using robust {{MDPs}}.
\newblock In \emph{International {{Conference}} on {{Artificial Intelligence}}
  and {{Statistics}} ({{AIStats}})}, 2021.

\bibitem[{Ben-Tal} et~al.(2009){Ben-Tal}, El~Ghaoui, and
  Nemirovski]{Ben-Tal2009}
{Ben-Tal}, A., El~Ghaoui, L., and Nemirovski, A.
\newblock \emph{Robust {{Optimization}}}.
\newblock {Princeton University Press}, 2009.

\bibitem[Bertsimas et~al.(2021)Bertsimas, {den Hertog}, and
  Pauphilet]{Bertsimas2021}
Bertsimas, D., {den Hertog}, D., and Pauphilet, J.
\newblock Probabilistic {{Guarantees}} in {{Robust Optimization}}.
\newblock \emph{SIAM Journal on Optimization}, 31\penalty0 (4):\penalty0
  2893--2920, 2021.

\bibitem[Brown et~al.(2020)Brown, Niekum, and Petrik]{Brown2020}
Brown, D.~S., Niekum, S., and Petrik, M.
\newblock Bayesian robust optimization for imitation learning.
\newblock In \emph{Advances in {{Neural Information Processing Systems}}
  ({{NeurIPS}})}, 2020.

\bibitem[Calafiore \& Campi(2005)Calafiore and Campi]{Calafiore2005}
Calafiore, G. and Campi, M.~C.
\newblock Uncertain convex programs: Randomized solutions and confidence
  levels.
\newblock \emph{Mathematical Programming, Series A}, 102:\penalty0 25--46,
  2005.

\bibitem[Cheng et~al.(2022)Cheng, Xie, Jiang, and Agarwal]{Cheng2022}
Cheng, C.-A., Xie, T., Jiang, N., and Agarwal, A.
\newblock Adversarially trained actor critic for offline reinforcement
  learning, 2022.

\bibitem[David \& Nagaraja(2003)David and Nagaraja]{David2003}
David, H.~A. and Nagaraja, H.~N.
\newblock \emph{Order Statistics}.
\newblock {John Wiley \& Sons, Ltd}, 3 edition, 2003.

\bibitem[Deisenroth et~al.(2021)Deisenroth, Faisal, and Ong]{Deisenroth2021}
Deisenroth, M.~P., Faisal, A.~A., and Ong, C.~S.
\newblock \emph{Mathematics for Machine Learning}.
\newblock Cambridge University Press, 2021.

\bibitem[Delage \& Mannor(2010)Delage and Mannor]{Delage2009}
Delage, E. and Mannor, S.
\newblock Percentile optimization for {{Markov}} decision processes with
  parameter uncertainty.
\newblock \emph{Operations Research}, 58\penalty0 (1):\penalty0 203--213, 2010.

\bibitem[Follmer \& Schied(2016)Follmer and Schied]{Follmer2016}
Follmer, H. and Schied, A.
\newblock \emph{Stochastic Finance: Introduction in Discrete Time}.
\newblock {De Gruyter Graduate}, fourth edition, 2016.

\bibitem[Gelman et~al.(2014)Gelman, Carlin, Stern, and Rubin]{Gelman2014}
Gelman, A., Carlin, J.~B., Stern, H.~S., and Rubin, D.~B.
\newblock \emph{Bayesian {{Data Analysis}}}.
\newblock {Chapman and Hall/CRC}, 3\textsuperscript{rd} edition, 2014.

\bibitem[Ghosh et~al.(2022)Ghosh, Ajay, Agrawal, and Levine]{Ghosh2022}
Ghosh, D., Ajay, A., Agrawal, P., and Levine, S.
\newblock Offline {{RL}} policies should be trained to be adaptive, 2022.

\bibitem[Gupta(2019)]{Gupta2019}
Gupta, V.
\newblock Near-optimal bayesian ambiguity sets for distributionally robust
  optimization.
\newblock \emph{Management Science}, 65\penalty0 (9):\penalty0 4242--4260,
  2019.

\bibitem[Hau et~al.(2023)Hau, Petrik, and Ghavamzadeh]{Hau2023}
Hau, J.~L., Petrik, M., and Ghavamzadeh, M.
\newblock Entropic risk optimization in discounted {{MDPs}}.
\newblock In \emph{Artificial {{Intelligence}} and {{Statistics}}
  ({{AISTATS}})}, 2023.

\bibitem[Hong et~al.(2023)Hong, Kveton, Katariya, Zaheer, and
  Ghavamzadeh]{Hong2023}
Hong, J., Kveton, B., Katariya, S., Zaheer, M., and Ghavamzadeh, M.
\newblock Multi-task off-policy learning from bandit feedback.
\newblock In \emph{{International Conference on Machine Learning}}, 2023.

\bibitem[Horn \& Johnson(2013)Horn and Johnson]{Horn2013}
Horn, R.~A. and Johnson, C.~A.
\newblock \emph{Matrix Analysis}.
\newblock Cambridge University Press, second edition, 2013.

\bibitem[Javed et~al.(2021)Javed, Brown, Sharma, Zhu, Balakrishna, Petrik,
  Dragan, and Goldberg]{Javed2021}
Javed, Z., Brown, D., Sharma, S., Zhu, J., Balakrishna, A., Petrik, M., Dragan,
  A., and Goldberg, K.
\newblock Policy gradient {{Bayesian}} robust optimization for imitation
  learning.
\newblock In \emph{International {{Conference}} on {{Machine Learning}}
  ({{ICML}})}, 2021.

\bibitem[Jin et~al.(2019)Jin, Netrapalli, Ge, Kakade, and Jordan]{jin2019short}
Jin, C., Netrapalli, P., Ge, R., Kakade, S.~M., and Jordan, M.~I.
\newblock A short note on concentration inequalities for random vectors with
  subgaussian norm.
\newblock \emph{arXiv preprint arXiv:1902.03736}, 2019.

\bibitem[Jin et~al.(2022)Jin, Yang, and Wang]{Jin2022}
Jin, Y., Yang, Z., and Wang, Z.
\newblock Is pessimism provably efficient for offline {{RL}}?, 2022.

\bibitem[Kitahara \& Tsuchiya(2018)Kitahara and Tsuchiya]{Kitahara2018}
Kitahara, T. and Tsuchiya, T.
\newblock An extension of {{Chubanov}}'s polynomial-time linear programming
  algorithm to second-order cone programming.
\newblock \emph{Optimization Methods and Software}, 33\penalty0 (1):\penalty0
  1--25, January 2018.
\newblock ISSN 1055-6788.
\newblock \doi{10.1080/10556788.2017.1382495}.

\bibitem[Lange et~al.(2012)Lange, Gabel, and Riedmiller]{Lange2012}
Lange, S., Gabel, T., and Riedmiller, M.
\newblock Batch reinforcement learning.
\newblock In \emph{Reinforcement Learning}, pp.\  45--73. Springer, 2012.

\bibitem[Lattimore \& Szepesvari(2018)Lattimore and Szepesvari]{Lattimore2018}
Lattimore, T. and Szepesvari, C.
\newblock \emph{Bandit Algorithms}.
\newblock Cambridge University Press, 2018.

\bibitem[Laurent \& Massart(2000)Laurent and Massart]{Laurent2000}
Laurent, B. and Massart, P.
\newblock Adaptive estimation of a quadratic functional by model selection.
\newblock \emph{The Annals of Statistics}, 28\penalty0 (5):\penalty0
  1302--1338, 2000.

\bibitem[Lobo et~al.(2023)Lobo, Cousins, Petrik, and Zick]{Lobo2023}
Lobo, E., Cousins, C., Petrik, M., and Zick, Y.
\newblock Percentile criterion optimization in offline reinforcement learning.
\newblock In \emph{Neural {{Information Processing Systems}} ({{NeurIPS}})},
  2023.

\bibitem[Lobo et~al.(2020)Lobo, Ghavamzadeh, and Petrik]{Lobo2020}
Lobo, E.~A., Ghavamzadeh, M., and Petrik, M.
\newblock Soft-{{Robust Algorithms}} for {{Handling Model Misspecification}},
  2020.

\bibitem[Lubin et~al.(2023)Lubin, Dowson, Garcia, Huchette, Legat, and
  Vielma]{Lubin2023}
Lubin, M., Dowson, O., Garcia, J.~D., Huchette, J., Legat, B., and Vielma,
  J.~P.
\newblock Jump 1.0: Recent improvements to a modeling language for mathematical
  optimization.
\newblock \emph{Mathematical Programming Computation}, 2023.

\bibitem[Luedtke \& Ahmed(2008)Luedtke and Ahmed]{Luedtke2008}
Luedtke, J. and Ahmed, S.
\newblock A sample approximation approach for optimization with probabilistic
  constraints.
\newblock \emph{SIAM Journal on Optimization}, 19\penalty0 (2):\penalty0
  674--699, 2008.

\bibitem[Nemirovski \& Shapiro(2006)Nemirovski and Shapiro]{Nemirovski2006}
Nemirovski, A. and Shapiro, A.
\newblock Scenario {{Approximations}} of {{Chance Constraints}}.
\newblock In Calafiore, G. and Dabbene, F. (eds.), \emph{Probabilistic and
  Randomized Methods for Design under Uncertainty}, pp.\  3--47. {Springer},
  2006.

\bibitem[Nemirovski \& Shapiro(2007)Nemirovski and Shapiro]{Nemirovski2007}
Nemirovski, A. and Shapiro, A.
\newblock Convex approximations of chance constrained programs.
\newblock \emph{SIAM Journal on Optimization}, 17\penalty0 (4):\penalty0
  969--996, 2007.

\bibitem[Petrik \& Russel(2019)Petrik and Russel]{Petrik2019}
Petrik, M. and Russel, R.~H.
\newblock Beyond confidence regions: Tight {{Bayesian}} ambiguity sets for
  robust {{MDPs}}.
\newblock In \emph{Advances in {{Neural Information Processing Systems}}},
  volume~32, 2019.

\bibitem[Pint{\'e}r(1989)]{Pinter1989}
Pint{\'e}r, J.
\newblock Deterministic approximations of probability inequalities.
\newblock \emph{Zeitschrift f\"ur Operations-Research}, 33\penalty0
  (4):\penalty0 219--239, 1989.

\bibitem[Rashidinejad et~al.(2022)Rashidinejad, Zhu, Ma, Jiao, and
  Russell]{Rashidinejad2022}
Rashidinejad, P., Zhu, B., Ma, C., Jiao, J., and Russell, S.
\newblock Bridging offline reinforcement learning and imitation learning: A
  tale of pessimism.
\newblock \emph{IEEE Transactions on Information Theory}, 68\penalty0
  (12):\penalty0 8156--8196, 2022.

\bibitem[Rasmussen \& Williams(2006)Rasmussen and Williams]{Rasmussen2006}
Rasmussen, C.~E. and Williams, C.
\newblock \emph{Gaussian Processes for Machine Learning}.
\newblock MIT Press, 2006.

\bibitem[Rockafellar \& Wets(2009)Rockafellar and Wets]{Rockafellar2009}
Rockafellar, R.~T. and Wets, R.~J.
\newblock \emph{Variational Analysis}.
\newblock Springer, 2009.

\bibitem[Shapiro et~al.(2014)Shapiro, Dentcheva, and Ruszczynski]{Shapiro2014}
Shapiro, A., Dentcheva, D., and Ruszczynski, A.
\newblock \emph{Lectures on Stochastic Programming: Modeling and Theory}.
\newblock {SIAM}, 2014.

\bibitem[Steimle et~al.(2021)Steimle, Kaufman, and Denton]{Steimle2021a}
Steimle, L.~N., Kaufman, D.~L., and Denton, B.~T.
\newblock Multi-model {{Markov}} decision processes.
\newblock \emph{IISE Transactions}, 53, 2021.

\bibitem[Su \& Petrik(2023)Su and Petrik]{Su2023}
Su, X. and Petrik, M.
\newblock Solving multi-model {{MDPs}} by coordinate ascent and dynamic
  programming,.
\newblock In \emph{Uncertainty in {{Artificial Intelligence}} ({{UAI}})}, 2023.

\bibitem[Uehara \& Sun(2023)Uehara and Sun]{Uehara2023}
Uehara, M. and Sun, W.
\newblock Pessimistic model-based offline reinforcement learning under partial
  coverage.
\newblock In \emph{International {{Conference}} on {{Learning Representations}}
  ({{ICLR}})}, 2023.

\bibitem[Vaart(2000)]{Vaart2000}
Vaart, A. W. V.~D.
\newblock \emph{Asymptotic {{Statistics}}}.
\newblock {Cambridge University Press}, 2000.

\bibitem[Vershynin(2010)]{vershynin2010introduction}
Vershynin, R.
\newblock Introduction to the non-asymptotic analysis of random matrices.
\newblock \emph{arXiv preprint arXiv:1011.3027}, 2010.

\bibitem[Xiao et~al.(2021)Xiao, Wu, Littlemore, Dai, Mei, Li, Szepesvari, and
  Schuurmans]{Xiao2021b}
Xiao, C., Wu, Y., Littlemore, T., Dai, B., Mei, J., Li, L., Szepesvari, C., and
  Schuurmans, D.
\newblock On the optimality of batch policy optimization algorithms.
\newblock In \emph{International {{Conference}} of {{Machine Learning}}
  ({{ICML}})}, 2021.

\bibitem[Xie et~al.(2021)Xie, Cheng, Jiang, Mineiro, and Agarwal]{Xie2021}
Xie, T., Cheng, C.-A., Jiang, N., Mineiro, P., and Agarwal, A.
\newblock Bellman-consistent pessimism for offline reinforcement learning.
\newblock In \emph{Advances in {{Neural Information Processing Systems}}},
  volume~34, pp.\  6683--6694, 2021.

\bibitem[Xie et~al.(2022)Xie, Bhardwaj, Jiang, and Cheng]{Xie2022c}
Xie, T., Bhardwaj, M., Jiang, N., and Cheng, C.-A.
\newblock {{ARMOR}}: A model-based framework for improving arbitrary baseline
  policies with offline data, 2022.

\end{thebibliography}


\clearpage
\onecolumn
\appendix


\section{Technical Background and Lemmas}
\label{app:background}

A scalar random variable $\tilde{x}\colon \Omega \to \Real$ with mean $\mu = \Ex{\tilde{x}}$ is sub-Gaussian with a variance factor $\sigma^2 \geq 0$ when
\begin{equation} 
\label{eq:scalar-rv-subgaussian}
\E \big[\exp\big(\lambda (\tilde{x} - \mu)\big)\big]
\; \le\;
\exp\big(\lambda^2 \sigma^2/2\big), \qquad  \forall \lambda \in  \Real \,.
\end{equation}
A {\em multivariate} random variable $\bm{\tilde{x}}\colon \Omega \to \Real^d$ with a mean $\bm{\mu} = \Ex{\bm{\tilde{x}}}$ is \emph{sub-Gaussian} with a covariance factor $\bm{\Sigma}\in \Real^{d\times d}$ when~\citep{vershynin2010introduction,jin2019short}
\begin{equation} 
\label{eq:multi-rv-subgaussian}
  \E \big[\exp\big(\lambda \bm{w}\tr(\bm{\tilde{x}} - \bm{\mu})\big)\big]
  \; \le\;
  \exp\big(\lambda^2 \bm{w}\tr \bm{\Sigma} \bm{w} / 2\big),
  \qquad \forall \lambda \in \Real, \; \forall \bm{w}\in \probs{d}\,.
\end{equation}

The Entropic Value at Risk~(EVaR) is a risk measure related to VaR, defined as~\cite{Ahmadi-Javid2012}
\begin{equation}\label{eq:evar-def-app}
  \evar{\alpha}{\tilde{x}}
  \; =\;
  \inf_{\beta>0} \, \beta^{-1} \big(\E[\exp(\beta\tilde{x})] - \log(1-\alpha)\big) \,, \qquad \forall \alpha \in [0,1).
\end{equation}

The following lemma shows that EVaR is an upper bound on VaR. This is a property that will be useful in our proofs later on.
\begin{lemma}
\label{lem:var-evar}
For any random variable $\tilde{x}\colon \Omega \to \Real $, we have that
  \[
    \var{\alpha}{\tilde{x}}
    \; \le\;
    \evar{\alpha}{\tilde{x}}, \qquad \forall \alpha \in [0,1).
  \]
\end{lemma}
\begin{proof}
This is a consequence of Proposition~3.2 in \citet{Ahmadi-Javid2012} and the fact that CVaR upper bounds VaR. 
\end{proof}

Similar to~\eqref{eq:var-normal} for VaR, we can show that for Gaussian random variables, $\tilde{x}  \sim \mathcal{N}(\mu, \sigma^2)$, EVaR has the following analytical form~\cite{Ahmadi-Javid2012}: 
\begin{equation} \label{eq:evar-normal}
\evar{\alpha}{\tilde{x}} \;=\; \mu + \sigma \cdot \sqrt{-2 \log (1-\alpha)}.
\end{equation}
One advantage of EVaR over VaR is that we can bound it in the more general case of sub-Gaussian random variables by the same bound as for a normal random variable in~\eqref{eq:evar-normal} (see the following lemma).
\begin{lemma} 
\label{lem:evar-subgaussian}
Let $\tilde{x}\colon \Omega \to \Real$ be a sub-Gaussian random variable defined according to~\eqref{eq:scalar-rv-subgaussian}. Then, we have
\[
 \evar{\alpha}{\tilde{x}} \;\le\; \mu + \sigma \cdot \sqrt{-2 \log (1-\alpha)}\,, \qquad \forall \alpha \in [0,1)\;. 
\]
\end{lemma}
\begin{proof}
From the translation invariance of EVaR~\citep[Theorem~3.1]{Ahmadi-Javid2012} and the definitions in~\eqref{eq:scalar-rv-subgaussian} and~\eqref{eq:evar-def-app}, we have 
\begin{align*}
\evar{\alpha}{\tilde{x}} &= \mu + \evar{\alpha}{\tilde{x} - \mu}
= \mu + \inf_{\beta>0} \, \beta^{-1} \cdot \Big(\Ex{\exp\big(\beta \cdot (\tilde{x} - \mu) \big)} - \log(1-\alpha)\Big) \\
&\le \mu + \inf_{\beta>0} \, \beta^{-1} \cdot \left(\frac{\beta^2 \sigma^2}{2}  - \log(1-\alpha)\right) = \mu + \sigma \cdot \sqrt{-2 \log (1-\alpha)}\,.
\end{align*}
The last step follows by solving for the optimal $\beta\opt = \sigma^{-1} \sqrt{-2 \log (1-\alpha)}$ from the first-order optimality conditions of the convex objective function.
\end{proof}


\section{Proofs of Section \ref{sec:algo}}
\label{sec:algo-proofs}

\begin{proof}[Proof of \cref{lem:regret-normal}]
We obtain by algebraic manipulation that
\begin{align*}
\max_{a\in \mathcal{A}} \; r(a; \bm{\tilde{\theta}}_D) - r(\bm{\pi}; \bm{\tilde{\theta}}_D)
  &= \max_{a\in \mathcal{A}} \; \bm{1}_a\tr \bm{\Phi}\tr \bm{\tilde{\theta}}_D - \bm{\pi}\tr \bm{\Phi}\tr  \bm{\tilde{\theta}}_D
  = \max_{a\in \mathcal{A}} \; \bm{1}_a\tr \left( \bm{\Phi}\tr \bm{\tilde{\theta}}_D -  \bm{1} \bm{\pi}\tr \bm{\Phi}\tr  \bm{\tilde{\theta}}_D \right) \\
   &=\max_{a\in \mathcal{A}} \; \bm{1}_a\tr \left( \bm{I} -  \bm{1} \bm{\pi}\tr \right) \bm{\Phi}\tr  \bm{\tilde{\theta}}_D \,. 
\end{align*}
Let $\bm{\tilde{x}}^{\bm{\pi}} = \left( \bm{I} -  \bm{1} \bm{\pi}\tr \right) \bm{\Phi}\tr  \bm{\tilde{\theta}}_D$, which is a linear transformation of the normal random variable $\bm{\tilde{\theta}}_D$. The result follows because linear transformations preserve normality~\cite{Deisenroth2021}.
\end{proof}

\subsection{Proof of \cref{thm:var-max-combined}}
We first report a result in the following lemma which we later use to prove \cref{thm:var-max-combined}.

\begin{lemma} \label{lem:var-max-sum-better}
Suppose $\bm{\tilde{x}}\colon \Omega \to \Real^k$ is a random variable such that all $\alpha$-quantiles, $\forall \alpha \in [0,1)$, for each $\tilde{x}_a, a\in \mathcal{A}$ are unique. Then, the following inequality holds for each $\alpha\in [0,1)$:
 \[
   \var{\alpha}{\max_{a\in \mathcal{A}} \tilde{x}_a}
   \;\le\; 
   \inf \left\{ \max_{a\in \mathcal{A}} \var{1-\xi_a}{\tilde{x}_a} \mid
     \bm{\xi}\in \Real^k_+, \, \bm{1}\tr \bm{\xi} = 1-\alpha \right\}\,.
 \]
We interpret the maximum of all $-\infty$ as $-\infty$.
\end{lemma}

\begin{proof}
The result develops as
\begin{align*}
\var{\alpha}{\max_{a \in  \mathcal{A}} \tilde{x}_a}
  &\overset{\text{(a)}}{=}  \sup\left\{ t \in \Real \mid \P{\max_{a \in \mathcal{A}} \tilde{x}_a \ge t} > 1 - \alpha\right\} \\
  &\overset{\text{(b)}}{\le}  \sup\left\{ t \in \Real \mid \P{\max_{a \in \mathcal{A}} \tilde{x}_a \ge t} \ge 1 - \alpha\right\} \\
  &\overset{\text{(c)}}{\le}  \sup\left\{ t \in \Real \mid \sum_{a \in \mathcal{A}} \P{ \tilde{x}_a \ge t} \ge  1 - \alpha \right\} \\
&\overset{\text{(d)}}{=}  \inf   \left\{ \sup\left\{ t \in \Real \mid \sum_{a \in \mathcal{A}} \P{ \tilde{x}_a \ge t} \ge \sum_{a\in \mathcal{A}} \xi_a \right\} \mid  \bm{\xi}\in \Real_+^k, \bm{1}\tr \bm{\xi} = 1 - \alpha \right\} \\
&\overset{\text{(e)}}{\le}  \inf   \left\{ \max_{a\in \mathcal{A}} \sup\left\{ t \in \Real \mid \P{ \tilde{x}_a \ge t} \ge \xi_a \right\} \mid  \bm{\xi}\in \Real_+^k, \bm{1}\tr \bm{\xi} = 1 - \alpha \right\} \\
&\overset{\text{(f)}}{\le}  \inf   \left\{ \max_{a\in \mathcal{A}} \var{1- \xi_a}{\tilde{x}_a} \mid  \bm{\xi}\in \Real_{+}^k, \bm{1}\tr \bm{\xi} = 1 - \alpha \right\} .
\end{align*}
{\em (a)} is from the definition of VaR. {\em (b)} follows by relaxing the set by replacing the strict inequality with a non-strict one. {\em (c)} follows by relaxing the constraint further using the union bound. {\em (d)} follows from algebraic manipulation because the objective is constant in the choice of $\bm{\xi}$. {\em (e)} holds by relaxing the sum constraints and then representing the supremum over a union of sets by a maximum of the suprema of the sets as
\begin{align*}
  \sup\left\{ t \in \Real \mid \sum_{a \in \mathcal{A}} \P{ \tilde{x}_a \ge t} \ge \sum_{a\in \mathcal{A}} \xi_a \right\}
  &\le
  \sup\left\{ t \in \Real \mid \P{ \tilde{x}_a \ge t} \ge \xi_a,\, \exists a\in \mathcal{A} \right\} \\
  &=
    \max_{a\in \mathcal{A}} \sup\left\{ t \in \Real \mid \P{ \tilde{x}_a \ge t} \ge \xi_a \right\}\,.
\end{align*}
Finally, {\em (f)} follows from the definition of VaR and because then the quantiles are unique~\cite{Follmer2016}
\[
  \var{1- \xi_a}{\tilde{x}_a} = \sup\left\{ t \in \Real \mid \P{ \tilde{x}_a \ge t} \ge \xi_a \right\} = \sup\left\{ t \in \Real \mid \P{ \tilde{x}_a \ge t} > \xi_a \right\}\,.
\]
The first equality is the definition of the upper quantile $q^+$ and the second equality is the definition of the lower quantile $q^-$, which are equal by the uniqueness assumption. 
\end{proof}

We are now ready to prove \cref{thm:var-max-combined}.
\begin{proof}[Proof of \cref{thm:var-max-combined}]
The first inequality in~\eqref{eq:var-union-normal} follows from \cref{lem:regret-normal} and \cref{lem:var-max-sum-better} by some algebraic manipulation. The second inequality in~\eqref{eq:var-union-normal} follows from upper bounding the VaR of a Gaussian random variable using~\eqref{eq:var-normal} and the fact that $\tilde{x}^{\bm{\pi}}_a$ is a Gaussian random variable with mean $\bm{\mu}\tr \bm{\Phi}(\bm{1}_a - \bm{\pi})$ and standard deviation $\|\bm{\Phi}(\bm{1}_a - \bm{\pi})\|_{\bm{\Sigma}}$.

The inequality $z_{1-\delta\xi_a} \le \sqrt{2 \log \nicefrac{1}{\delta\xi_a} }$ holds because for a standard normal random variable $\tilde{y}$, we have that
  \[
    z_{1-\delta\xi_a} = \var{1-\delta\xi_a}{\tilde{y}}
    \overset{\text{(a)}}{\le} \evar{1-\delta\xi_a}{\tilde{y}}
    \overset{\text{(b)}}{=} \sqrt{2 \log(1/\delta\xi_a)}\,.
  \]
{\em (a)} follows from \cref{lem:var-evar} and {\em (b)} is by~\eqref{eq:evar-normal}.
\end{proof}

\subsection{Proof of \cref{thm:lower bound}}

First, we prove a lower bound on the VaR of a single Gaussian random variable. 
\begin{lemma} \label{lem:var-single-lower}
Suppose that $\tilde{x} \sim \mathcal{N}(0,1)$ and $\alpha \ge \frac{1}{2}$. Then
\[
    \var{\alpha}{\tilde{x}} 
    \; \ge\;
    -1 + \sqrt{1 - \log (\sqrt{2\pi}) - 2\log (1-\alpha) }.
  \]
\end{lemma}
\begin{proof}
To establish this lower bound on VaR, we use the known bounds on the cumulative distribution function of a Gaussian random variable as stated, for example, in eq. (13.1) in \citet{Lattimore2018}. For any $t\in \Real$ we have that
\[
  \P{\tilde{x} \ge t}
  \;\ge\; 
  \frac{\sqrt{8\pi^{-1}}}{2|t| + \sqrt{4 t^2 + 16}}
  \exp \left( - \frac{t^2}{2} \right).
\]
From the definition of VaR in~\eqref{eq:var-definition-sup} we get that
\begin{align*}
\var{\alpha}{\tilde{x}}
  &= \sup\left\{ t \in \Real \mid \P{\tilde{x} \ge t} > 1 - \alpha\right\} \\
  &= \sup\left\{ t \in \Real_+ \mid \P{\tilde{x} \ge t} > 1 - \alpha\right\} \\
  &\ge \sup\left\{ t \in \Real_+ \mid   \frac{\sqrt{8\pi^{-1}}}{2 t + \sqrt{4 t^2 + 16}}
  \exp \left( - \frac{t^2}{2} \right) > 1 - \alpha\right\} \\
  &\ge \sup\left\{ t \in \Real_+ \mid   \frac{\sqrt{8\pi^{-1}}}{4(t+1)}
  \exp \left( - \frac{t^2}{2} \right) > 1 - \alpha\right\}.
\end{align*}
Here, we restricted $t$ to be non-negative, which does not impact the VaR value because for $\alpha \ge  0.5$ we have that $\var{\alpha}{\tilde{x}} \ge 0$. The first inequality is a lower bound that follows by tightening the feasible set in the supremum. The final inequality follows since $\sqrt{4t^2 + 16} \le 2t + 4$ from the triangle inequality.

Then, algebraic manipulation of the right-hand side above gives us that
\begin{equation*}
\var{\alpha}{\tilde{x}}
  \ge \sup \left\{ t\in \Real_+ \mid -t^2 - 2t > 1 \log (1-\alpha) + 2\log \sqrt{2\pi} \right\}.
\end{equation*}
Then, using the fact that the constraint is concave in $t$, we get the final lower bound on VaR by solving the quadratic equation.
\end{proof}

The following lemma bounds the VaR of a maximum of independent random variables. This is possible because the maximum is the first \emph{order statistic} which has an easy-to-represent CDF~\cite{David2003}.

\begin{lemma} \label{lem:var-multi-variable}
  Suppose that $\tilde{x}_i\colon \Omega \to \Real, i = 1, \dots , n$ are i.i.d. random variables. Then
  \[
    \var{\alpha}{\max_{i =1, \dots , n} \tilde{x}_i}
    \;=\; 
    \var{\alpha^{\nicefrac{1}{n}}}{\tilde{x}_1}. 
  \]
\end{lemma}
\begin{proof}
 Recall i.i.d. random variables satisfy that
  \[
    \P{\max_{i=1, \dots , n} \tilde{x}_i}
    = \prod_{i=1, \dots , n} \P{\tilde{x}_i}
    = \P{\tilde{x}_1}^n.
  \]
  The result then follows from the definition of VaR in~\eqref{eq:var-definition} and from algebraic manipulation as
  \begin{align*}
    \var{\alpha}{\max_{i=1, \dots , n} \tilde{x}_i}
    &= \inf \left\{ t \in \Real \mid \P{\max_{i=1, \dots , n} \tilde{x}_i > t} \le 1 - \alpha \right\} 
    = \inf \left\{ t \in \Real \mid \P{\max_{i=1, \dots , n} \tilde{x}_i  \le t} \ge \alpha \right\} \\
    &= \inf \left\{ t \in \Real \mid \P{\tilde{x}_1 \le t}^n \ge \alpha \right\} 
    = \inf \left\{ t \in \Real \mid \P{\tilde{x}_1 \le t} \ge \alpha^{\nicefrac{1}{n}} \right\} = \var{\alpha^{\nicefrac{1}{n}}}{\tilde{x}_1}.
  \end{align*}
\end{proof}

\begin{proof}[Proof of \cref{thm:lower bound}]
Define a restricted set of actions $\mathcal{A}_2 = \mathcal{A} \setminus \left\{ a_1 \right\}$. As in the remainder of the paper, we use $\alpha  = 1- \delta$ to simplify the notation in this proof. 

From the definition of regret in~\eqref{eq:regret-high-confidence-VaR} and the monotonicity of VaR~\cite{Shapiro2014} we get that the regret of the $\bm{\pi}$ can be lower bounded as the maximum regret compared only to actions in $\mathcal{A}_2$:
\begin{align*}
  \mathfrak{R}_{\delta}(\bm{\pi}) &= \var{\alpha}{\max_{a\in \mathcal{A}} r(\bm{\tilde{\theta}}, a) - r(\bm{\tilde{\theta}}, a_1) }
\ge \var{\alpha}{\max_{a\in \mathcal{A}_2} r(\bm{\tilde{\theta}}, a) - r(\bm{\tilde{\theta}}, a_1) }.
\end{align*}
From the theorem's assumptions, the random variables $\tilde{z}_{a} = r(\bm{\tilde{\theta}}, a) - r(\bm{\tilde{\theta}}, a_1)$ for $a\in \mathcal{A}_2$ are independent and identically distributed as $\mathcal{N}(\mu_2 - \mu_1, \sigma_2^2 + \sigma_1^2)$ where $\sigma_i = \Sigma_{i,i}$ for $i = 1, \dots , k$. Then, using the inequality above and \cref{lem:var-multi-variable} we get that
\begin{align*}
  \mathfrak{R}_{\delta}(\bm{\pi})
  &= \var{\alpha}{\max_{a\in \mathcal{A}_2} r(\bm{\tilde{\theta}}, a) - r(\bm{\tilde{\theta}}, a_1) }
  \ge \var{1-\alpha^{\nicefrac{1}{k}}}{\tilde{z}} \\
  &= (\mu_2 - \mu_1) + \sqrt{\sigma_1^2 + \sigma_2^2} \cdot  \var{1-\alpha^{\nicefrac{1}{k}}}{\frac{\tilde{z}_{a_2} - (\mu_2 - \mu_1)}{\sqrt{\sigma_1^2 + \sigma_2^2}} }.
\end{align*}
Here, we used the fact that VaR is positively homogenous and translation equivariant. The result follows by \cref{lem:var-single-lower} since the random variable inside of the VaR above is distributed as $\mathcal{N}(0,1)$.
\end{proof}

\subsection{Proof of \cref{thm:var-max-l2}}

This result follows from standard robust optimization techniques~(see, for example, \citet{Gupta2019, Petrik2019}) as well as bandit analysis. In fact, similar or perhaps almost identical analysis has been used to analyze the regret of \texttt{FlatOPO} in \citet{Hong2023}. We provide an independent proof for the sake of completeness.

The following two auxiliary lemmas are used to show that a robust optimization over a credible region can be used to upper bound the VaR of any random variable. The first auxiliary lemma establishes a sufficient condition for a robust optimization being an overestimate of VaR.

\begin{lemma} \label{lem:sufficiently-robust}
Suppose that we are given an ambiguity set $\mathcal{P} \subseteq \mathcal{X}$, a function $g\colon \mathcal{X} \to \Real$, and a random variable $\bm{\tilde{x}}:\Omega \to \mathcal{X}$. If $\mathcal{P} \cap \mathcal{Z} \neq \emptyset$ for 
$\mathcal{Z} = \big\{ \bm{x} \in \mathcal{X} \mid g(\bm{x}) \ge \var{\alpha}{g(\bm{\tilde{x}})} \big\}$, then
  \[
    \var{\alpha}{g(\bm{\tilde{x}})} \; \le \; \sup_{\bm{x}\in \mathcal{P}} \; g(\bm{x}) \,.
  \]
\end{lemma}

\begin{proof}
By the hopothesis, there exists some $\bm{\hat{x}} \in \mathcal{P} \cap \mathcal{Z}$. Then, we have
$\sup_{\bm{x}\in \mathcal{P}} g(\bm{x}) \; \ge \; g(\bm{\hat{x}}) \; \ge \; \var{\alpha}{g(\bm{\tilde{x}})}$ that concludes the proof, where the first inequality is by definition and the second one is from the definition of the set $\mathcal Z$.
\end{proof}

The second auxiliary lemma shows that a credible region is sufficient to upper bound VaR using a robust optimization problem.

\begin{lemma} \label{prop:credible-upper}
Suppose that we are given an ambiguity set $\mathcal{P} \subseteq \mathcal{X}$, a function $g\colon \mathcal{X} \to \Real$, and a random variable $\bm{\tilde{x}}:\Omega \to \mathcal{X}$. Then, we have
  \[
    \Pr{\bm{\tilde{x}} \in  \mathcal{P}} \ge \alpha
    \qquad \Longrightarrow \qquad 
  \var{\alpha}{g(\bm{\tilde{x}})} \;\le\; \sup_{\bm{x}\in \mathcal{P}} g(\bm{x})\,. 
 \]
\end{lemma}

\begin{proof}
Our proof is by contradiction using \cref{lem:sufficiently-robust}. We start by assuming that $\Pr{\bm{\tilde{x}} \in  \mathcal{P}}$. Define $\mathcal{Z} = \big\{ \bm{x} \in \mathcal{X} \mid  g(\bm{x}) \ge \var{\alpha}{g(\bm{\tilde{x}})} \big\}$ as in \cref{lem:sufficiently-robust}. From \cref{lem:sufficiently-robust}, we know that if $\sup_{\bm{x}\in \mathcal{P}} g(\bm{x}) \;\ge\; \var{\alpha}{g(\bm{\tilde{x}})}$ is false, then we should have $\mathcal{P} \cap \mathcal{Z} = \emptyset$. By the definition of VaR, we have that $\Pr{\bm{\tilde x} \in \mathcal{Z}} > 1- \alpha$.  Then, we get a contradiction with $\mathcal{P} \cap \mathcal{Z} = \emptyset$ as follows
  \[
   1 \ge \Pr{\bm{\tilde{x}} \in \mathcal{P} \cup \mathcal{Z}} = \Pr{\bm{\tilde{x}} \in \mathcal{P}} + \Pr{\bm{\tilde{x}} \in \mathcal{Z}} > \alpha + 1 - \alpha > 1\,.
 \]
\end{proof}

The following lemma uses a standard technique for constructing a credible region for a multivariate normal distribution~\cite{Hong2023, Gupta2019}.

\begin{lemma} \label{lem:normal-credible}
Suppose that $\bm{\tilde{x}} \sim \mathcal{N}(\bm{\mu}, \bm{\Sigma})$ is a multi-variate normal random variable with a mean $\bm{\mu}\in \Real^d$ and a covariance matrix $\bm{\Sigma} \in \Real^{d \times d}$. Then the set $\mathcal{P} \subseteq \Real^d$, defined as
  \[
   \mathcal{P} =  \left\{ \bm{x}\in \Real^d \mid
    \| \bm{x} - \bm{\mu}\|_{\bm{\Sigma}^{-1}}^2 \le \chi_d^2(\alpha) \right\},
  \]
with $\chi_d^2(\alpha)$ being the $\alpha$-quantile of the $\chi^2_d$ distribution, satisfies that
\(
  \Pr{\bm{\tilde{x}} \in \mathcal{P} } = \alpha \,.\)
\end{lemma}
\begin{proof}
One can readily verify that $\bm{\Sigma}^{{-\frac{1}{2}}} (\bm{\tilde{x}} - \bm{\mu}) \sim \mathcal{N}(\bm{0}, \bm{I})$ is a standard multivariate normal distribution. The norm of this value is a sum of i.i.d. standard normal variables, and thus, is distributed according to the $\chi^2_d$ distribution with $d$ degrees of freedom:
\[
\left(\bm{\Sigma}^{{-\frac{1}{2}}} (\bm{\tilde{x}} - \bm{\mu})\right)\tr \left( \bm{\Sigma}^{{-\frac{1}{2}}} (\bm{\tilde{x}} - \bm{\mu})\right) =  \| \bm{\tilde{x}} - \bm{\mu} \|_{\bm{\Sigma}^{-1}}^{2} \sim \chi^2_d\,.
\]
Therefore, by algebraic manipulation and the definition of a quantile, we obtain that
\[
  \Pr{\bm{\tilde{x}} \in \mathcal{P} } = \Pr{ \| \bm{\tilde{x}} - \bm{\mu}\|_{\bm{\Sigma}^{-1}}^2 \le \chi_d^2(\alpha)  } = \alpha \,.
\]
\end{proof}

Finally, the following lemma derives the optimal solution of a quadratic optimization problem that arises in the formulation. 
\begin{lemma} \label{lem:quadratic-solution}
The equality
\begin{equation}
\label{eq:temp00}
\max_{\bm{p}\in\mathbb R^d} \left\{ \bm{x}\tr \bm{p} \mid \left\| \bm{p} - \bm{\hat{p}} \right\|_{\bm{C}}^2 \le b,\, \bm{p}\in \Real^k  \right\} \; =\; \bm{x}\tr \bm{\hat{p}} + \sqrt{b} \cdot \| \bm{x} \|_{\bm{C}^{-1}}  \,
\end{equation}
holds for any given vectors $\bm{x}, \bm{\hat{p}}\in \Real^d$ and a matrix $\bm{C} \in \Real^{d\times d} $ that is positive definite: $\bm{C} \succ \bm{0}$. 
\end{lemma}
\begin{proof}
From the convexity of the optimization problem in~\eqref{eq:temp00}, we can construct the optimizer $\bm{p}\opt$ using KKT conditions as
\[
\bm{p}\opt = \bm{\hat{p}} + \sqrt{b} \cdot \| \bm{x} \|_{\bm{C}^{-1}} \cdot \bm{C}^{-1} \bm{x}\,.
\]
The result then follows by substituting $\bm{p}\opt$ into the maximization problem in the lemma.
\end{proof}

We are now ready to prove the main theorem.
\begin{proof}[Proof of \cref{thm:var-max-l2}]
We derive the bound in~\eqref{eq:var-max-l2-eq} using the robust representation of VaR~\cite{Ben-Tal2009}. We first construct the set $\mathcal{P}_{\delta} \subseteq \Real^d$ as  
\begin{equation} 
\label{eq:robust-set}
\mathcal{P}_\delta \;=\;  \left\{ \bm{\theta} \in \Real^d \mid \| \bm{\theta} - \bm{\mu} \|^{2}_{\bm{\Sigma}^{-1}} \le \chi^2_d(1-\delta) \right\} \,.
\end{equation}
Using \cref{lem:normal-credible} and the definition of $\mathcal{P}_{\delta}$ in~\eqref{eq:robust-set}, we can see that $\mathcal{P}_{\delta}$ is indeed a credible region:
  \[
    \Pr{\bm{\tilde{\theta}} \in  \mathcal{P}_{\delta}} = 1-\delta\,.
  \]
Then, \cref{prop:credible-upper} gives us the first inequality in~\eqref{eq:var-max-l2-eq}:
  \[
  \mathfrak{R}_{\delta}(\bm{\pi}) 
  \le
    \max_{\bm{\theta}\in \mathcal{P}_{\delta}} \max_{a\in \mathcal{A}} \left( r(a; \bm{\theta}) - r(\bm{\pi}; \bm{\theta}) \right) \,.
 \] 
The second inequality in~\eqref{eq:var-max-l2-eq} is a consequence of \cref{lem:quadratic-solution} with $\bm{x} = \bm{\Phi}(\bm{1}_a - \bm{\pi})$, $\bm{\hat{p}} = \bm{\mu}$, $\bm{p} = \bm{\theta}$, $\bm{C} = \bm{\Sigma}^{-1}$, and $b = \sqrt{\chi^2(1-\delta)}$. 

Finally, the inequality $\sqrt{\chi^2_d(1-\delta)} \le  \sqrt{5 d \log(1/\delta)}$ follows from Lemma~1 in \citet{Laurent2000} as in the proof of Lemma~3 in \citet{Hong2023}.
\end{proof}

\subsection{Proof of \cref{cor:algo-is-great}}


\begin{proof}
The corollary is an immediate consequence of \cref{thm:var-max-combined,thm:var-max-l2} and the construction of \cref{alg:BRMOB}. By construction, $\bm{\pi}^0$ is the solution to  
\[
  \bm{\pi}^0 \in \arg \min_{\bm{\pi}\in \probs{k}} \max_{a\in \mathcal{A}} \, \bm{\mu}\tr \bm{\Phi} (\bm{1}_a - \bm{\pi}) + \norm{ \bm{\Phi} (\bm{1}_a - \bm{\pi})}_{\bm{\Sigma}} \cdot  \nu_a^0  \,,
\]
where $\nu_a^0$ is defined in \cref{alg:BRMOB}. Therefore, using \cref{thm:var-max-combined,thm:var-max-l2} to upper bound $\nu_a$, we obtain
\[
\mathfrak{R}_{\delta}(\bm{\pi}^0) \le \min_{\bm{\pi}\in \probs{k}} \max_{a\in \mathcal{A}} \, \bm{\mu}\tr \bm{\Phi} (\bm{1}_a - \bm{\pi}) + \norm{ \bm{\Phi} (\bm{1}_a - \bm{\pi})}_{\bm{\Sigma}} \cdot \min \left\{\sqrt{2 \log (k/\delta) } , \sqrt{5 d \log(1/\delta) } \right\}\,.
\]
This proves the corollary when $i\opt = 0$ in \cref{alg:BRMOB}. Then, using \cref{thm:var-max-combined} with general $\bm{\xi}$, we observe that the algorithm selects $i\opt > 0$ only when $\mathfrak{R}_{\delta}(\bm{\pi}^{i\opt}) \le \rho^{i\opt} \le \rho^0$, which means that the corollary also holds.
\end{proof}


\section{Proofs of Section \ref{sec:bayes-regr-analys}} \label{sec:theory-proofs}

\subsection{Proof of \cref{thm:regret-bound}}

\begin{proof}
To prove the first claim of the theorem, let $\bm{\bar{\pi}}$ be a policy that minimizes the linear component of the regret:
  \[
    \bm{\bar{\pi}} \in \arg \min_{\bm{\pi}\in \probs{k}}  \bm{\mu}\tr \bm{\Phi} (\bm{1}_a - \bm{\pi}) \,.
  \]
Note that the minimum above is upper-bounded by 0. Next we use \cref{cor:algo-is-great} to bound the regret:
\begin{align*}
    \mathfrak{R}_{\delta}(\bm{\hat\pi})
    &\le \min_{\bm{\pi}\in \probs{k}} \max_{a\in \mathcal{A}} \, \bm{\mu}\tr \bm{\Phi} (\bm{1}_a - \bm{\pi}) + \norm{ \bm{\Phi} (\bm{1}_a - \bm{\pi})}_{\bm{\Sigma}_n} \cdot  \min \left\{\sqrt{2 \log (\nicefrac{k}{\delta}) } , \sqrt{5 d \log(\nicefrac{1}{\delta}) } \right\}  \\
    &\le \max_{a\in \mathcal{A}} \, \bm{\mu}\tr \bm{\Phi} (\bm{1}_a - \bm{\bar{\pi}}) + \norm{ \bm{\Phi} (\bm{1}_a - \bm{\bar{\pi}})}_{\bm{\Sigma}_n} \cdot  \min \left\{\sqrt{2 \log (\nicefrac{k}{\delta}) } , \sqrt{5 d \log(\nicefrac{1}{\delta}) } \right\}  \\
    &\le \max_{a\in \mathcal{A}} \, \norm{ \bm{\Phi} (\bm{1}_a - \bm{\bar{\pi}})}_{\bm{\Sigma}_n} \cdot  \min \left\{\sqrt{2 \log (\nicefrac{k}{\delta}) } , \sqrt{5 d \log(\nicefrac{1}{\delta}) } \right\} .
\end{align*}
Now, we bound the term $\norm{ \bm{\Phi} (\bm{1}_a - \bm{\bar{\pi}})}_{\bm{\Sigma}_n}$. Recall that $\| \bm{\bar{\pi}} \|_2 \le \| \bm{\bar{\pi}} \|_1 \le 1 $, since $\bm{\bar{\pi}} \in \probs{k}$. Then, for each $a\in \mathcal{A}$, we have by algebraic manipulation that
\begin{align*}
    \norm{ \bm{\Phi} (\bm{1}_a - \bm{\bar{\pi}})}_{\bm{\Sigma}_n}^2
    &\; =\;  (\bm{1}_a - \bm{\bar{\pi}})\tr \bm{\Phi}\tr \bm{\Sigma}_n \bm{\Phi} (\bm{1}_a - \bm{\bar{\pi}}) \\
    &\; =\;   \bm{1}_a\tr \bm{\Phi}\tr \bm{\Sigma}_n \bm{\Phi} \bm{1}_a  +
      \bm{\bar{\pi}}\tr \bm{\Phi}\tr \bm{\Sigma}_n \bm{\Phi} \bm{\bar{\pi}}
    - 2 \cdot \bm{1}_a\tr \bm{\Phi}\tr \bm{\Sigma}_n \bm{\Phi} \bm{\bar{\pi}}\\
    &\; \overset{\text{(a)}}{\le}\;  4 \max_{a'\in \mathcal{A}}  \bm{1}_{a'}\tr \bm{\Phi}\tr \bm{\Sigma}_n \bm{\Phi} \bm{1}_{a'}
      \; =\;  4 \max_{a'\in \mathcal{A}}  \bm{\phi}_{a'}\tr \bm{\Sigma}_n \bm{\phi}_{a'} \,.
\end{align*}
{\em (a)} holds by the Cauchy-Schwartz inequality because
\[
  -\bm{1}_a\tr \bm{\Phi}\tr \bm{\Sigma}_n \bm{\Phi} \bm{\bar{\pi}}
  \; \le\; 
  \|\bm{\Sigma}_n^{\nicefrac{1}{2}}\bm{\Phi}\bm{1}_a\|_2 \|\bm{\Sigma}_n^{\nicefrac{1}{2}}\bm{\Phi} \bm{\bar{\pi}}\|_2
  \; \le\;
  \max_{a'\in \mathcal{A}}   \|\bm{\Sigma}_n^{\nicefrac{1}{2}}\bm{\Phi}\bm{1}_{a'}\|_2^2\,.
\]
The last inequality in the above equation is satisfied because 
$\|\bm{\Sigma}_n^{\nicefrac{1}{2}}\bm{\Phi} \bm{\bar{\pi}}\|_2 \le \sum_{a'\in \mathcal{A}} \bar{\pi}_{a'}   \|\bm{\Sigma}_n^{\nicefrac{1}{2}}\bm{\Phi}\bm{1}_{a'}\|_2 \le  \max_{a'\in \mathcal{A}}  \|\bm{\Sigma}_n^{\nicefrac{1}{2}}\bm{\Phi}\bm{1}_{a'}\|_2$, which in turn follows by Jensen's inequality from the convexity of the $\ell_2$-norm and the fact that $\bm{\bar{\pi}}\in \probs{k}$. The term $\bm{\bar{\pi}}\tr \bm{\Phi}\tr \bm{\Sigma}_n \bm{\Phi} \bm{\bar{\pi}}$ is upper bounded by an analogous argument.

Now \cref{asm:good-data} implies the following for each $a\in \mathcal{A}$:
\begin{align}
    \bm{G}_n &\quad \succeq \quad \gamma n \cdot \bm{\phi}_a \bm{\phi}_a\tr \nonumber \\
    \bm{\Sigma}_0^{-1} + \bar{\sigma}^{-2}\bm{G}_n &\quad\succeq \quad \bm{\Sigma}_0^{-1} + \bar{\sigma}^{-2} \cdot  \gamma n \cdot \bm{\phi}_a \bm{\phi}_a\tr \quad\succ\quad \bm{0} \nonumber \\
    (\bm{\Sigma}_0^{-1} + \bar{\sigma}^{-2}\bm{G}_n)^{-1} &\quad\preceq \quad (\bm{\Sigma}_0^{-1} + \bar{\sigma}^{-2} \cdot\gamma n \cdot \bm{\phi}_a \bm{\phi}_a\tr)^{-1} \nonumber \\
    \bm{\phi}_a\tr (\bm{\Sigma}_0^{-1} + \bar{\sigma}^{-2}\bm{G}_n)^{-1}  \bm{\phi}_a &\quad\le \quad \bm{\phi}_a\tr(\bm{\Sigma}_0^{-1} + \bar{\sigma}^{-2} \cdot\gamma n \cdot \bm{\phi}_a \bm{\phi}_a\tr)^{-1}  \bm{\phi}_a \nonumber \\
    \bm{\phi}_a\tr  \bm{\Sigma}_n \bm{\phi}_a &\quad\le\quad \bm{\phi}_a\tr(\bm{\Sigma}_0^{-1} + \bar{\sigma}^{-2} \cdot\gamma n \cdot \bm{\phi}_a \bm{\phi}_a\tr)^{-1}  \bm{\phi}_a.
\label{eq:temp0}    
\end{align}
The second line holds because we assumed $\bm{\Sigma}_0 \succ 0$, and thus, $\bm{\Sigma}_0^{-1} \succ 0$, and adding a positive definite matrix preserves definiteness. The third line holds from the definiteness in the second line and \citet[corollary 7.7.4(a)]{Horn2013}. Finally, the fourth line holds from the definition of positive semi-definiteness.

We continue by applying the Woodbury matrix identity to~\eqref{eq:temp0}, which give us the following inequality for each $a\in \mathcal{A}$:
\begin{align*}
    \bm{\phi}_a\tr \bm{\Sigma}_n \bm{\phi}_a
    &\quad \le\quad   \bm{\phi}_a\tr(\bm{\Sigma}_0^{-1} + \bar{\sigma}^{-2} \cdot\gamma n \cdot \bm{\phi}_a \bm{\phi}_a\tr)^{-1}  \bm{\phi}_a 
    \quad =\quad    \frac{1}{(\bm{\phi}_a\tr \bm{\Sigma}_0 \bm{\phi}_a)^{-1} + \bar{\sigma}^{-2} \cdot\gamma n} \\
    &\quad \le\quad   \frac{1}{\lambda_{\max}(\bm{\Sigma}_0)^{-1}  + \bar{\sigma}^{-2} \cdot\gamma n} \, ,
\end{align*}
where $\lambda_{\max}$ computes the maximum eigenvalues of the matrix. The inequality above holds because
\[
0 \le \bm{\phi}_a\tr \bm{\Sigma}_0 \bm{\phi}_a \le \lambda_{\max}(\bm{\Sigma}_0) \| \bm{\phi}_a \|,
\]
which can be seen from the eigendecomposition of the symmetric matrix. Substituting the inequality above proves the theorem. 

To prove the special case of the theorem with $\bm{\mu}_n = \bm{0}$, let $\bm{\pi}^0$ be the solution in the first iteration of \cref{alg:BRMOB}. Given the posterior distribution of $\bm{\tilde{\theta}}_D$, the policy $\bm{\pi}^0$ is chosen as
\begin{align*}
\bm{\pi}^0 &\in \arg \min_{\bm{\pi}\in \probs{k}} \max_{a\in \mathcal{A}} \; \bm{0}\tr \bm{\Phi} (\bm{1}_a - \bm{\pi}) + \norm{ \bm{\Phi} (\bm{1}_a - \bm{\pi})}_{\bm{\Sigma}} \cdot  \nu_a^0  \\
&= \arg \min_{\bm{\pi}\in \probs{k}} \max_{a\in \mathcal{A}} \; \norm{ \bm{\Phi} (\bm{1}_a - \bm{\pi})}_{\bm{\Sigma}} \,. 
\end{align*}
The square of this minimization problem can be formulated as a convex quadratic program
\begin{equation} \label{eq:quadratic-program}
\min_{t\in \Real, \; \bm{\pi}\in \probs{k}} \left\{ t \mid t \ge \norm{\bm{\Sigma}_n^{\nicefrac{1}{2}}\bm{\phi}_a  - \bm{\Sigma}^{\nicefrac{1}{2}}\bm{\Phi} \bm{\pi} }_2^2 \; , \, \forall a\in \mathcal{A} \right\}\,.
\end{equation}
Because  $\bm{\Sigma}^{\nicefrac{1}{2}}\bm{\Phi} \bm{\pi} \in \Real^d$ and is a convex combination of points in $\Real^d$, there exists an optimal $\bm{\pi}^0$ such that $l = |\left\{ a\in \mathcal{A} \mid \pi^0_a > 0 \right\}| \le d+1$~\cite{Rockafellar2009}. Then, let $\hat{a} \in \arg\max_{a'\in \mathcal{A}} \pi^0_{a'}$. We have that  $\pi^0_{\hat{a}} \ge \frac{1}{l}$ because $l$ actions are positive, and the constraint $t \ge \norm{\bm{\Sigma}_n^{\nicefrac{1}{2}}\bm{\phi}_a  - \bm{\Sigma}^{\nicefrac{1}{2}}\bm{\Phi} \bm{\pi} }_2^2$  is active (holds with equality). If the constraint were not active, this would be a contradiction with the optimality of $\bm{\pi}^0$ because decreasing $\pi_{\hat{a}}^0$ would reduce the objective. Then, using the inequalities above and the triangle inequality, we get that the optimal $t\opt $ in~\eqref{eq:quadratic-program} satisfies
\begin{align*}
\sqrt{t\opt} &=  \norm{\bm{\Sigma}_n^{\nicefrac{1}{2}}\bm{\phi}_{\hat{a}}  - \bm{\Sigma}^{\nicefrac{1}{2}}\bm{\Phi} \bm{\pi}^0 }_2 
= \left(1-\max_{a''\in \mathcal{A}} \pi^0_{a''}\right) \norm{\bm{\Sigma}_n^{\nicefrac{1}{2}}\bm{\phi}_{\hat{a}} - \bm{\Sigma}_n^{\nicefrac{1}{2}}\bm{\phi}_{a'}}_2 \\
&\le  \left(1-\max_{a''\in \mathcal{A}} \pi^0_{a''}\right) \norm{\bm{\Sigma}_n^{\nicefrac{1}{2}}\bm{\phi}_{\hat{a}} - \bm{\Sigma}_n^{\nicefrac{1}{2}}\bm{\phi}_{a'}}_2 
\le  2\left(1- \max_{a''\in \mathcal{A}}\pi^0_{a''}\right) \norm{\bm{\Sigma}_n^{\nicefrac{1}{2}}\bm{\phi}_{a'}}_2 .
\end{align*}
The remainder of the proof follows from the same steps as the proof of \cref{thm:regret-bound}. The lower bound on $\max_{a'\in \mathcal{A}} \hat{\pi}_{a'}$ holds from the existence of $\pi^0$ with at most $d+1$ positive elements, as discussed above.
\end{proof}

\subsection{Proof of \cref{thm:counterexample}}

\begin{proof}[Proof of \cref{thm:counterexample}]
  First, from the construction of \cref{exm:lcb-counterexample}, we have that
  \[
   a_1 \in  \arg\min_{a\in \mathcal{A}} \mu_a - \beta \cdot \sigma_{a} = \arg\min_{a\in \mathcal{A}} \beta \cdot \sigma_{a}- \beta \cdot \sigma_{a} = \mathcal{A}, 
 \]
 and therefore $\bm{\pi}_{\mathrm{LCB}}$ is the policy returned by LCB that breaks ties as specified.
 Then, using \cref{thm:lower bound}, we bound the regret of LCB as
 \[
   \mathfrak{R}_{\delta}(\bm{\pi}_{\mathrm{LCB}})
   \; \ge \;
   \mu_{a_2} + \sigma_{a_2} \cdot \kappa_{\mathrm{l}}(k-1) =
   \beta \cdot \sigma_{a_2} + \sigma_{a_2} \cdot \kappa_{\mathrm{l}}(k-1) =
   (\beta + \kappa_{\mathrm{l}}(k-1)) \cdot  \sigma_{a_2}.
 \]

 In contrast, \texttt{Greedy} selects $a_2$ deterministically since
 \[
  a_2 \in \arg\min_{a\in \mathcal{A}} \mu_a = \left\{ a_2, \dots , a_k \right\}.
\]
Then, using \cref{thm:var-max-combined} and~\eqref{eq:var-union-normal-simple} in particular, we upper bound the regret of $\bm{\pi}_{\mathrm{G}}$ as
\begin{align*}
  \mathfrak{R}_{\delta}(\bm{\pi}_{\mathrm{G}})
  &\le \max_{a\in \mathcal{A}} \mu_a^{\bm{\pi}_{\mathrm{G}}} + \sigma_a^{\bm{\pi}_{\mathrm{G}}} \cdot \kappa_{\mathrm{u}}(k) \\
  &= \max_{a\in \left\{ a_2, \dots ,a_k \right\}} \mu_a^{\bm{\pi}_{\mathrm{G}}} + \sigma_a^{\bm{\pi}_{\mathrm{G}}} \cdot \kappa_{\mathrm{u}}(k) \\
  &= \max_{a\in \left\{ a_2, \dots ,a_k \right\}} \sqrt{ \sigma_a^{2} + \sigma_{a_2}^{2} } \cdot \kappa_{\mathrm{u}}(k) \\
  &= \max_{a\in \left\{ a_2, \dots ,a_k \right\}} \sqrt{2}\cdot \sigma_{a_2} \cdot \kappa_{\mathrm{u}}(k).
\end{align*}
The equalities follow from substituting the definitions of relative means and variances and from algebraic manipulation.
\end{proof}

\section{Sub-Gaussian Posterior} 
\label{sec:sub-gauss-post}

We discuss here how our results can extend to $\bm{\tilde{\theta}}_D$ with sub-Gaussian distributions. The modifications necessary are quite minor. The key to the approach is to generalize \cref{thm:var-max-combined} to a sub-Gaussian distribution as the following theorem states.

\begin{theorem}\label{thm:var-max-combined-subg}
Suppose that $\bm{\tilde{\theta}}_D$ is a random variable with an atomless distribution that is sub-Gaussian with mean $\bm{\mu}$ and covariance factor $\bm{\Sigma}$. Then the regret for each $\bm{\pi} \in \probs{k}$ satisfies that
\begin{equation} \label{eq:var-union-subg}
  \begin{aligned}
  \mathfrak{R}_{\delta}(\bm{\pi}) 
  &\;\le\;  \min_{\bm{\xi}\in \probs{k}}  \max_{a\in \mathcal{A}} \,  \var{1-\delta\xi_a}{ r(a; \bm{\tilde{\theta}}_D) - r(\bm{\pi}; \bm{\tilde{\theta}}_D) } \\
  &\;\le\;  
    \min_{\bm{\xi}\in \probs{k}} \max_{a\in \mathcal{A}} \, \bm{\mu}\tr \bm{\Phi} (\bm{1}_a - \bm{\pi})+ \norm{ \bm{\Phi} (\bm{1}_a - \bm{\pi})}_{\bm{\Sigma}} \cdot \sqrt{2 \log (\nicefrac{1}{\delta\xi_a}) }.
  \end{aligned}
\end{equation}
\end{theorem}
\begin{proof}
The first inequality in~\eqref{eq:var-union-subg} holds by \cref{thm:var-max-combined} since this inequality does not require that the posterior is normal. That is, we have that
\begin{align*}
  \mathfrak{R}_{\delta}(\bm{\pi}) 
    &\;\le\;  \min_{\bm{\xi}\in \probs{k}}  \max_{a\in \mathcal{A}} \,  \var{1-\delta\xi_a}{ r(a; \bm{\tilde{\theta}}_D) - r(\bm{\pi}; \bm{\tilde{\theta}}_D) } \\
    &\;=\;  \min_{\bm{\xi}\in \probs{k}}  \max_{a\in \mathcal{A}} \,  \var{1-\delta\xi_a}{ (\bm{1}_a - \bm{\pi})\tr \bm{\Phi}\tr \bm{\tilde{\theta}}_D } \\
    &\;\le\;  \min_{\bm{\xi}\in \probs{k}}  \max_{a\in \mathcal{A}} \,  \evar{1-\delta\xi_a}{ (\bm{1}_a - \bm{\pi})\tr \bm{\Phi}\tr \bm{\tilde{\theta}}_D } \,.
\end{align*}
The last inequality follows from \cref{lem:var-evar}. For each $a\in \mathcal{A}$, the definition of a multi-variate sub-Gaussian random variable in~\eqref{eq:multi-rv-subgaussian} with $\bm{w}\tr  = (\bm{1}_a - \bm{\pi})\tr \bm{\Phi}\tr$ implies that that $(\bm{1}_a - \bm{\pi})\tr \bm{\Phi}\tr \bm{\tilde{\theta}}_D$ is sub-Gaussian with mean $\mu = (\bm{1}_a - \bm{\pi})\tr \bm{\Phi}\tr \bm{\mu}$ and a variance factor $\sigma^2 = (\bm{1}_a - \bm{\pi})\tr \bm{\Phi}\tr \bm{\Sigma} \bm{\Phi} (\bm{1}_a - \bm{\pi})$. Therefore, from \cref{lem:evar-subgaussian} we have
\[
  \min_{\bm{\xi}\in \probs{k}}  \max_{a\in \mathcal{A}} \,  \evar{1-\delta\xi_a}{ (\bm{1}_a - \bm{\pi})\tr \bm{\Phi}\tr \bm{\tilde{\theta}}_D }
  \le
  \bm{\mu}\tr \bm{\Phi} (\bm{1}_a - \bm{\pi})+ \norm{ \bm{\Phi} (\bm{1}_a - \bm{\pi})}_{\bm{\Sigma}} \cdot \sqrt{2 \log (\nicefrac{1}{\delta\xi_a}) } \;,
\]
which proves the result.
\end{proof}

\Cref{thm:var-max-l2} can also be extended to the sub-Gaussian setting but seems to require an additional assumption that $\|\bm{\tilde{\theta}}-\bm{\mu}\|^2_{\bm{\Sigma}^{-1}}$ is a sub-gamma random variable, and we leave it for future work.

Armed with \cref{thm:var-max-combined-subg}, we can adapt \cref{alg:BRMOB} to the sub-Gaussian setting simply by setting $\nu^0_a = \sqrt{2 \log(\nicefrac{k}{\delta})}$. Note that~\eqref{eq:finetuning} already uses the correct inequality for a sub-Gaussian distribution.



\section{Other Objectives}
\label{sec:other-objectives}

We now briefly discuss two other related objectives as alternatives to minimizing the {\em high-confidence Bayesian regret}, defined in~\eqref{eq:regret-high-confidence def} and~\eqref{eq:regret-high-confidence-VaR}. These objectives may be preferable in some settings because they can be solved optimally using simple and tractable techniques. 


\subsection{Expected Bayes Regret} 
\label{sec:optim-bayes-regr}

The first objective we discuss is {\em expected Bayes regret}, which is obtained by simply replacing the VaR by expectation in~\eqref{eq:regret-high-confidence-VaR}. In this case, the goal of the agent is to minimizes the expected regret, defined as
\begin{equation*}
   \min_{\bm{\pi}\in \probs{k}} \Ex{ \max_{a\in \mathcal{A}} \; r(a; \bm{\tilde{\theta}}_D) - r(\bm{\pi}; \bm{\tilde{\theta}}_D) }\,.
\end{equation*}
Using the linearity property of the expectation operator and the reward function $r$, we have
\[
  \arg  \min_{\bm{\pi}\in \probs{k}} \Ex{ \max_{a\in \mathcal{A}} \; r(a; \bm{\tilde{\theta}}_D) - r(\bm{\pi}; \bm{\tilde{\theta}}_D) } =
  \arg  \max_{\bm{\pi}\in \probs{k}} \Ex{r(\bm{\pi}; \bm{\tilde{\theta}}_D) } = 
  \arg  \max_{\bm{\pi}\in \probs{k}} r\left(\bm{\pi}; \Ex{\bm{\tilde{\theta}}_D}\right) .
\]
This means it is sufficient to maximize the return for the mean posterior parameter value. In most case, such as when the posterior over $\bm{\tilde{\theta}}_D$ is normal, this is an easy optimization problem to solve optimally. 


\subsection{High-confidence Return}

The second objective we discuss is {\em high-confidence return}, which is obtained by simply replacing the regret with return in~\eqref{eq:regret-high-confidence-VaR}. In this case, the goal of the agent is to minimizes the VaR of the return random variable as
\begin{equation}\label{eq:return-hc}
   \min_{\bm{\pi}\in \probs{k}} \var{1-\delta}{  - r(\bm{\pi}; \bm{\tilde{\theta}}_D) }
   =
   \min_{\bm{\pi}\in \probs{k}} \var{1-\delta}{  - \bm{\pi}\tr \bm{\Phi}\tr \bm{\tilde{\theta}}_D) }\,.
\end{equation}
One may think of this objective as minimizing the regret with respect to $0$. The reward inside is negated because we use VaR which measures costs rather than rewards. Note that $-\var{1-\delta}{-\tilde{x}} \approx \var{\delta}{\tilde{x}}$ with an equality for atomless (continuous) distributions. 

When $\bm{\tilde{\theta}}_D \sim \mathcal{N}(\bm{\mu}, \bm{\Sigma})$, the optimization in~\eqref{eq:return-hc} can be solved \emph{optimally} using an LCB-style algorithm. Then, using the properties of linear transformation of normal distributions, for each $\bm{\pi} \in \probs{ k }$, we obtain 
\[
  \bm{\pi}\tr \bm{\Phi}\tr \bm{\tilde{\theta}}_D \sim  \mathcal{N}(\bm{\pi}\tr \bm{\Phi}\tr \bm{\mu},\, \bm{\pi}\tr \bm{\Phi}\tr \bm{\Sigma} \bm{\Phi} \bm{\pi} )\,.
\]
Combining the objective in~\eqref{eq:return-hc} with~\eqref{eq:var-normal}, we get that the objective is
\begin{equation} \label{eq:var-return-normal}
\max_{\bm{\pi}\in \probs{k}} \bm{\pi}\tr \bm{\Phi}\tr \bm{\mu}  - \sqrt{\bm{\pi}\tr \bm{\Phi}\tr \bm{\Sigma} \bm{\Phi} \bm{\pi}} \cdot z_{1-\delta } \,. 
\end{equation}
Recall that $z_{1-\delta}$ is the $1-\delta$-th quantile of the standard normal distribution. We can reformulate~\eqref{eq:var-return-normal} as the following second-order conic program (for $\delta \le \nicefrac{1}{2}$)
\begin{equation*} 
  \begin{mprog}
    \maximize{\bm{\pi}\in \Real^{k}, \; s\in \Real} \bm{\pi}\tr \bm{\Phi}\tr \bm{\mu} - z_{1-\delta }\cdot s
    \stc s^2 \le \bm{\pi}\tr \bm{\Phi}\tr \bm{\Sigma} \bm{\Phi} \bm{\pi} ,
    \cs \bm{1}\tr \bm{\pi } = 1, \quad \bm{\pi} \ge \bm{0}\,.
  \end{mprog}
\end{equation*}

When restricted to deterministic policies, the optimization in~\eqref{eq:var-return-normal} reduces to a plain deterministic LCB algorithm. The \texttt{FlatOPO} algorithm can be seen as an approximation of~\eqref{eq:var-return-normal} in which $z_{1-\delta}$ is replaced by its upper bound.


\section{Additional Experimental Details} \label{sec:addit-exper-results}

\begin{figure} 
  \centering
  \includegraphics[width=0.3\linewidth]{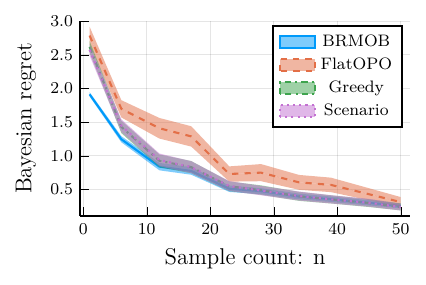}
  \hspace{0.02\linewidth}
  \includegraphics[width=0.3\linewidth]{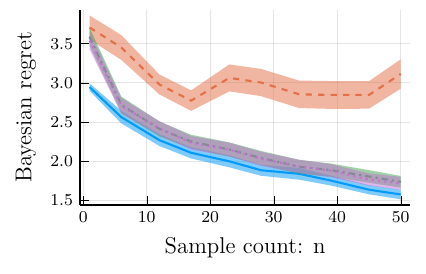}
  \hspace{0.02\linewidth}
  \includegraphics[width=0.3\linewidth]{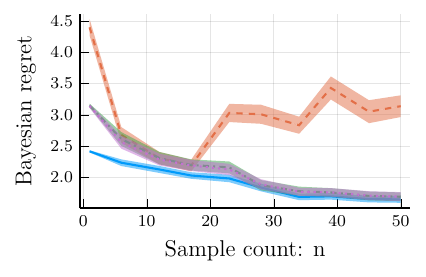}
  \caption{ Bayesian regret with $k=d=5$ {\em (left)}, $k=d=50$ {\em (middle and right)}. The prior mean is $\bm{\mu}_0 = \bm{0}$ {\em (left and middle)} and $(\bm{\mu}_0)_a = \sqrt{a}\;$ for $a = 1, \dots  50$ \emph{(right)}.}
\label{fig:independent-actions-ribbon}
\end{figure}

\begin{figure} 
  \centering
  \includegraphics[width=0.3\linewidth]{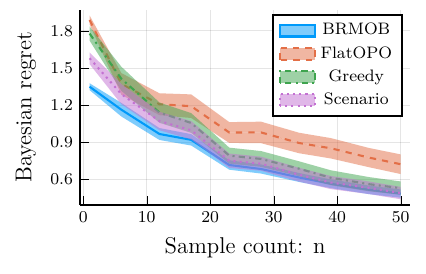}
  \hspace{0.02\linewidth}
  \includegraphics[width=0.3\linewidth]{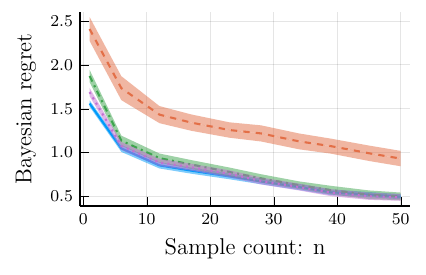}
  \hspace{0.02\linewidth}
  \includegraphics[width=0.3\linewidth]{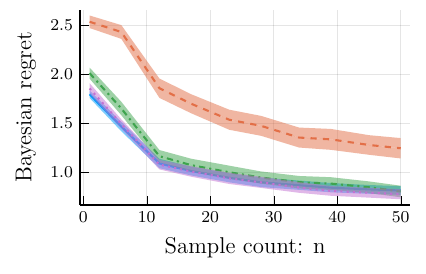}
  \caption{ Bayesian regret with $d=4$ and $k=10$ {\em (left)}, $k=50$ {\em (middle)}, and $k=100$ {\em (right)}.}
\label{fig:dependent-actions-ribbon}
\end{figure}

\begin{figure}
  \centering
  \includegraphics[width=0.3\linewidth]{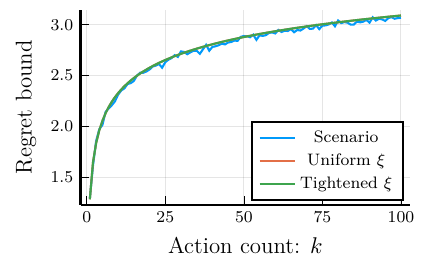}
  \hspace{0.02\linewidth}
  \includegraphics[width=0.3\linewidth]{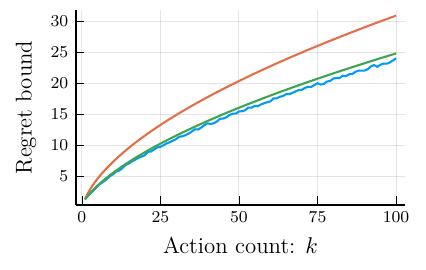}
  \hspace{0.02\linewidth}
  \includegraphics[width=0.3\linewidth]{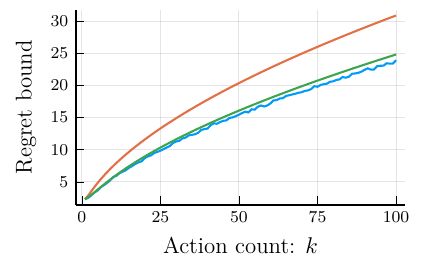}
  \caption{Regret bounds in \cref{thm:var-max-combined} for different choices of $\bm{\xi}$ as a function of $k$. The posterior distribution of $\bm{\tilde{x}}$ is normal with $\bm{\mu} = \bm{0},\, \bm{\Sigma} = \bm{I}$ \emph{(left)}, $\bm{\mu} = \bm{0},\, \Sigma_{aa} = a^2 / k$ \emph{(middle)}, and $\mu_a = a / k,\, \Sigma_{aa} = a^2 / k$ \emph{(right)} with $a = 1, \dots , k$.}
  \label{fig:bounds-demonstrations}
\end{figure}

In this section, we provide some additional experimental results. First, \cref{fig:independent-actions-ribbon,fig:dependent-actions-ribbon} report the same results as \cref{fig:independent-actions,fig:dependent-actions} but also report the 95\% confidence interval for the average regret over the 100 runs. Second, we report the effect of the tightening step on the quality of the bounds in \cref{fig:bounds-demonstrations} compared to a scenario-based estimation. In this simplified example, we fix some policy $\bm{\pi}$ and \emph{assume} the particular parameters of the distribution of $\tilde{x}_a = \bm{\tilde{\theta}}_D\tr \bm{\Phi} (\bm{1}_a - \bm{\pi}), a\in \mathcal{A}$, which is normal by \cref{lem:regret-normal}. The results in the figure show that when the distribution $\bm{\tilde{x}}$ is close to i.i.d. the tightening step does not improve the bound. This is expected since the optimal $\bm{\xi}$ in~\eqref{eq:finetuning} is nearly uniform. However, when the means or variances of the $\tilde{x}_a$ vary across actions $a\in \mathcal{A}$, then the tightening step can significantly reduce the error bound. 

\Cref{fig:runtime} compares the runtime of the algorithms considered as a function of the number of arms. The runtime excludes the time to compute the posterior distribution which is independent of the particular method considered. We use MOSEK to compute the SOCP optimization and do not run any tightening steps. The number of samples $m$ needed for the Scenario algorithm was derived from the Dvoretzky-Kiefer-Wolfowitz bound as
\[
 m\; =\; \frac{100}{(1-0.95)^2} \log \left( \frac{2 \cdot k}{0.05} \right). 
\]
This number of samples guarantees a small sub-optimality gap with probability $95\%$. We suspect, however, that this number of samples can be reduced with more careful assumptions and algorithmic design~\cite{Calafiore2005, Nemirovski2007, Nemirovski2006}. Such analysis is beyond the scope of this work.

\begin{figure}
  \centering
  \includegraphics[width=0.3\linewidth]{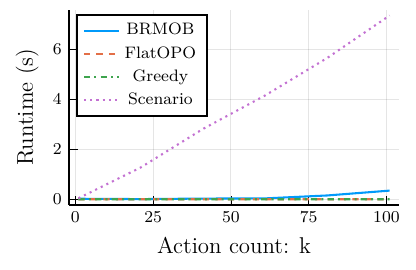}
  \caption{Runtime comparison of algorithms in seconds for a problem with $\bm{\mu} = \bm{0}$ and $\bm{\Sigma} =\bm{I}$.}
  \label{fig:runtime}
\end{figure}

\end{document}